\documentclass{article}

\PassOptionsToPackage{numbers, compress}{natbib}
\usepackage[preprint]{neurips_2026}

\usepackage[utf8]{inputenc} 
\usepackage[T1]{fontenc}    
\usepackage{hyperref}       
\usepackage{url}            
\usepackage{booktabs}       
\usepackage{amsfonts}       
\usepackage{nicefrac}       
\usepackage{microtype}      
\usepackage{xcolor}         
\usepackage{graphicx}
\usepackage{colortbl} 
\usepackage{siunitx}
\definecolor{urorow}{RGB}{238,234,246}

\usepackage{amsmath}
\usepackage{amssymb}
\usepackage{mathtools}
\usepackage{amsthm}
\usepackage{enumitem}
\usepackage{bbm}
\usepackage{siunitx}
\usepackage{cancel}

\usepackage{cleveref}
\crefformat{equation}{Eq.~(#2#1#3)}
\Crefformat{equation}{Eq.~(#2#1#3)}
\crefrangeformat{equation}{Eqs.~(#3#1#4--#5#2#6)}
\Crefrangeformat{equation}{Eqs.~(#3#1#4--#5#2#6)}
\crefmultiformat{equation}{Eqs.~(#2#1#3)}{ and~(#2#1#3)}{, (#2#1#3)}{ and~(#2#1#3)}
\Crefmultiformat{equation}{Eqs.~(#2#1#3)}{ and~(#2#1#3)}{, (#2#1#3)}{ and~(#2#1#3)}
\crefalias{appendix}{section}
\Crefname{appendix}{Appendix}{Appendices}
\crefname{table}{tab.}{tabs.}
\Crefname{table}{Tab.}{Tabs.}

\newtheorem{theorem}{Theorem}[section]

\newcommand{\pitheta}{\pi_{\theta}}
\newcommand{\piref}{\pi_{\text{ref}}}

\title{Be Your Own Teacher: Steering Protein Language Models via Unsupervised Reward Optimization}

\author{
Lanqing Li\textsuperscript{\rm 1}\thanks{Corresponding Author}\, , \, 
Shentong Mo\textsuperscript{\rm 2},
Yang Yu\textsuperscript{\rm 3},
\textbf{Pheng-Ann Heng}\textsuperscript{\rm 1}
\\
\textsuperscript{\rm 1} The Chinese University of Hong Kong,\\
\textsuperscript{\rm 2} MBZUAI, 
\textsuperscript{\rm 3} Hong Kong University of Science and Technology\\
[.5em]
{\tt\small \{lanqingli1993, shentongmo\}@gmail.com, eeyangyu@ust.hk, pheng@cse.cuhk.edu.hk}
\setcounter{footnote}{0}
}

\begin{document}

\maketitle

\begin{abstract}
Protein language models (PLMs) have emerged as powerful tools for controllable biomolecular design, yet their post-training adaptation typically relies on costly wet-lab validation or curated preference datasets. To overcome this supervision bottleneck, we introduce unsupervised reward optimization of PLMs, a comprehensive framework for steerable protein generation without ground-truth labels. Our key insight is that task-agnostic rewards, which combine intrinsic model uncertainty with extrinsic semantic consistency informed by protein representation models, exhibit strong correlation with controllability measures across base models and temperature regimes. Building upon this discovery, we propose two offline algorithms: Soft Reward Optimization (SRO) and  Binarized Reward Optimization (BRO), which effectively maximize the classical RLHF objective induced by these proxy rewards. Extensive experiments on compositional out-of-distribution prompts demonstrate that both methods significantly outperform competitive baselines (DPO, KTO), while approaching oracle performance across multiple sampling temperatures, model scales and protein families. Moreover, PLMs fine-tuned with unsupervised rewards can achieve consistently higher coverage compared to their base model in pass@k evaluations. By enabling self-improvement of PLMs through their own generated experience, our framework provides a scalable pathway toward controllable biomolecular design in settings where labeled preferences or experimental feedback are scarce or unavailable.
\end{abstract}

\section{Introduction}\label{sec:intro}

Just as large language models (LLMs) encode, express and reason over natural languages of human intelligence, biomolecular foundation models (BFMs) have been developed~\cite{rives2019biological, alley2019unified,  9477085, jumper2021highly, baek2021accurate, chen2022interpretable,   madani2023large, watson2023novo, susaprot2024, nguyen2024sequence, hayes2025simulating, lewis2025scalable, passaro2025boltz, chen2025xtrimopglm, mo2026sadit} to decipher the language of life. Exemplified by protein language models (PLMs) which were trained on billions of biological sequences\footnote{By \textit{biological sequences}, we mean molecular representations that can be tokenized and serialized for sequential modeling. For instance, multimodal PLMs such as ESM3 employ different input tracks to encode protein sequence, structure, and function.} of proteins, BFMs learn expressive deep representations emerging within language models that reflect the evolution, structure and function of biomolecules without explicit supervision on those properties~\cite{heinzinger2019modeling,meier2021language,raotransformer2021}. Moreover, frontier PLMs like ProGen~\cite{madani2023large, nijkamp2023progen2} and ESM3~\cite{hayes2025simulating} enable controllable generation and design~\cite{ferruz2022controllable} by following prompts that specify protein families, functions, or other design objectives. It has been extensively shown that the generation quality and fidelity can be significantly improved with scale by post-training~\cite{hayes2025simulating, widatalla2024aligning, stocco2024guidinggenerativeproteinlanguage}.

Supervised fine-tuning (SFT) and reinforcement learning from human feedback (RLHF) have emerged as the dominant paradigm for aligning large language models (LLMs) with human preferences during post-training adaptation~\cite{ziegler2019fine, stiennon2020learning, bai2022training, ouyang2022training, korbak2023pretraining}. More recently, reinforcement learning with verifiable rewards (RLVR) has been popularized to achieve remarkable reasoning capability of LLMs on mathematics, coding and science benchmarks~\cite{shao2024deepseekmath, guo2025deepseek, team2025qwq, yang2025qwen3, comanici2025gemini}. However, these approaches typically incur high costs for data labeling, and when scaling beyond human expertise in specialized domains, obtaining reliable ground-truth supervision becomes increasingly infeasible~\cite{burnsweak, silver2025welcome, he2026how}. For post-training BFMs toward generalist biological artificial intelligence (GBAI)~\cite{rao2026generalist}, the challenge intensifies: data curation requires prohibitively expensive wet-lab validation or clinical trials and poses substantial safety risks in healthcare and biomedical applications.

In the domain of LLMs, this supervision bottleneck~\cite{he2026how} has spurred growing interest in unsupervised post-training, which leverages intrinsic and extrinsic reward signals without ground-truth labels, achieving encouraging gains in language tasks such as free-form question-answering and reasoning~\cite{zhaoabsolute, zuo2025ttrl, zhang2025right}. Yet in biological domains, how to devise reliable proxy rewards for BFM post-training and equally importantly, how to design algorithms robust to the presumably noisy supervision signals remain largely unexplored. 

To address the outstanding challenge of unsupervised post-training for steerable generation of BFMs, we take protein language models as testbeds and conduct a holistic study covering the design of rewards, data sampling and optimization algorithms as well as benchmark tasks for \textit{in silico} evaluation. By examining a broad spectrum of inherent and external statistical metrics across different sampling temperatures, we identify reward functions that exhibit good correlation with the PLM's ability to follow high-level functional and structural instructions. Based on the reward design,
we propose two offline algorithms, namely \textit{Soft Reward Optimization (SRO)} and \textit{Binarized Reward Optimization (BRO)} to optimize the classical RLHF objective, by leveraging continuous and discrete reward signals respectively. We find that both proposed methods significantly enhance the steerability of frontier PLMs despite relying solely on noisy supervision, and attain performance comparable to that of the oracle (base model fine-tuned with ground-truth labels) in certain temperature regimes. Our main contributions include:


\begin{itemize}[leftmargin=*, itemsep=2pt, topsep=2pt]
    \item We present, to the best of our knowledge, the first investigation to fine-tune protein language models (PLMs) with \textit{unsupervised rewards} for controllable generation. This work paves the way for steerable adaptation and self-improvement of PLMs on offline, unlabeled corpora, with promising implications for biomedical domains that suffer from limited annotated data.

    \item We provide a novel, comprehensive recipe for unsupervised post-training of PLMs, involving (i) task-agnostic metrics as proxy rewards, (ii) a multi-temperature sampling strategy for creating diverse training sets, and (iii) SRO and BRO algorithms that effectively enhance the instruction-following capability of frontier PLMs, even when trained with noisy supervision.
    \item Our empirical study demonstrates that under the same unsupervised reward labeling protocol, both SRO and BRO outperform competitive baselines like DPO~\cite{rafailov2023direct} and KTO~\cite{ethayarajh2024model} by a substantial margin across various out-of-distribution (OOD) generation tasks, sampling temperatures and model scales. 
    \item To facilitate future research, we provide novel datasets and benchmark tasks by constructing realistic compositional OOD prompts from 7 Pfam~\cite{finn2014pfam} protein families, to probe the controllability and generalization of PLMs in \textit{de novo} protein design~\cite{ferruz2022controllable, winnifrith2024generative, kortemme2024novo,yang2026past}.
\end{itemize}

\section{Related Work}\label{sec:relatedwork}

\textbf{Protein Language Models}\,\,\,  Inspired by the inherent similarity between natural languages and biological sequences, protein language models (PLMs) have revolutionized computational biology by learning rich, contextual representations from massive unlabeled corpora of protein sequences through NLP-inspired architectures~\cite{vaswani2017attention, devlin2019bert} and self-supervised pre-training~\cite{rives2019biological, 9477085, heinzinger2019modeling, meier2021language, raotransformer2021}. Pioneering models like ESM~\cite{rives2019biological} and ProtTrans~\cite{9477085} demonstrate that transformer architectures can implicitly capture evolutionary constraints, structural stability, and functional motifs. This foundational capability has catalyzed a new generation of controllable generative PLMs. For instance, ProGen~\cite{madani2023large, nijkamp2023progen2} and ProLLaMA~\cite{lv2025prollama} enable conditional generation to create novel proteins across diverse families, while frontier multimodal foundation model like ESM3~\cite{hayes2025simulating} can jointly reason over sequence, structure, and function to perform complex, instruction-guided design tasks, such as generating a functional fluorescent protein with low homology to wild types. Recent studies~\cite{widatalla2024aligning, stocco2024guidinggenerativeproteinlanguage, hayes2025simulating} apply preference-based~\cite{rafailov2023direct}, directed evolution~\cite{jiang2024rapid, yang2019machine} and activation steering~\cite{huang2025steering} methods to greatly enhance the steerability of PLMs. However, the reliance on curated datasets of preference pairs and experimental observations significantly constrains their applicability in biological domains. To this end, our work charts a course for PLM self-improvement through their own generated experience instead of human data~\cite{silver2025welcome}, thereby opening a promising avenue towards scientific superintelligence~\cite{burnsweak, song2025evaluating}.


\textbf{Unsupervised Fine-Tuning of Language Models}\,\,\, Aligning LLMs with complex primitives typically relies on Reinforcement Learning from Human Feedback (RLHF)~\cite{ziegler2019fine, stiennon2020learning, bai2022training, ouyang2022training}. However, such approach is severely hampered by the "supervision bottleneck"~\cite{he2026how}: acquiring high-quality labels incurs substantial human labor and even worse, requires costly and time-consuming wet-lab experiments or clinical data in biological domains.
This challenge has motivated the study of unsupervised post-training methods that use proxy rewards instead of ground-truth labels. A key line of research is Unsupervised Reinforcement Learning with Verifiable Rewards (URLVR), where existing methods can be classified as either \textit{intrinsic} or \textit{extrinsic} based on their source of supervision signals~\cite{he2026how}. Intrinsic reward methods such as those based on self-certainty~\cite{zhao2025learning}, entropy~\cite{zhang2025right, cui2025entropy, agarwal2025unreasonable}, or majority voting~\cite{zuo2025ttrl} leverage proxy rewards generated solely by the model itself. They operate through a common "sharpening" mechanism that reinforces the model's initial confident predictions, achieving early training gains at the risk of reward hacking and model collapse at later stages~\cite{agarwal2025unreasonable, shafayat2025can, zhang2025no}. In contrast, extrinsic reward methods
generate rewards via mechanisms such as self-supervised objectives on unlabeled corpora~\cite{dong2025reinforcement, hatamizadeh2025rlp, she2025dupo, akter2026nemotron} or external computation from compilers, proof assistants or game engines~\cite{zhaoabsolute, simonds2025ladder, shao2025deepseekmath, hubert2025olympiad}, which are entirely independent of the model's internal state. These rewards are generally more robust and scalable than intrinsic ones, however, their reliance on feedback from verifiable environments renders them infeasible in many scientific domains that require large-scale simulation or wet-lab/clinical validation. In this paper, we propose hybrid, task-agnostic reward functions for PLM fine-tuning that combine intrinsic uncertainty measures with self-consistency metrics from an external protein representation model ESMC~\cite{esm2024cambrian}. Our empirical study shows that this unique design exhibits consistent agreement with controllability measures (e.g. keyword/structure recovery) without \textit{any prior knowledge} of the ground-truth labels, and is demonstrated to elicit strong instruction-following capability of PLMs via reward optimization algorithms.

\section{Method}\label{sec:method}

\subsection{Preliminaries}
To achieve \textit{de novo} protein design via controllable generation, such as generating highly specialized functional proteins, frontier generative PLMs~\cite{madani2023large, hayes2025simulating, nijkamp2023progen2} are generally\footnote{Some multimodal PLMs like ESM3 may skip the SFT stage by tokenizing control tags (e.g. structural and functional tokens) in the pretraining stage. However, if the downstream task represents a novel distribution that cannot be naively characterized by known prompts during pretraining, SFT should still be used for adaptation to the new distribution and prompts/conditions.} trained via a three-stage procedure similar to modern LLMs~\cite{ouyang2022training}:

\textbf{Pretraining} \,\, Given a large corpus of protein sequences $y$, train the model $\pitheta$ to minimize a generative masked language modeling objective:
\vspace{-3pt}
\begin{align}
    \mathcal{L}_{\text{pretrain}} = -\mathbb{E}_{y, m}\sum_{i\in m}\log\pitheta(y_i|y_{\backslash m}).\label{eqn:pretrain_obj}
\end{align}
\vspace{-\baselineskip}

If $m$ is a causal mask, \cref{eqn:pretrain_obj} becomes the classical autoregressive next-token prediction. This stage endows the model with general knowledge of the statistical and evolutionary patterns of natural protein sequences. The pretrained model $\pi_0$ can therefore be used for \textit{unconditional generation}.


\textbf{Supervised Finetuning (SFT)} \,\, Finetune the model via masked language modeling on data that are more relevant to the downstream task. Analogous to instruction finetuning~\cite{weifinetuned} in LLM, SFT data often comprise control tags as instructions $x$ and a completed protein sequence as the response $y$. The finetuned model $\piref$ is therefore trained by the objective
\vspace{-3pt}
\begin{align}
    \mathcal{L}_{\text{SFT}} = -\mathbb{E}_{x, y, m}\sum_{i\in m}\log\pitheta(y_i|y_{\backslash m }, x)
\end{align}
\vspace{-\baselineskip}

and applicable to \textit{conditional generation}.

\textbf{Reinforcement Learning (RL)} \,\, To further align $\pi_\text{ref}$ for controllable generation and design, a reward/score function $r(x, y)$ operates to quantify the quality of the generated sequence $y$ given prompt $x$. Similar to RLHF/RLVR for LLM post-training, this RL stage can be formulated as solving a KL-constrained reward maximization problem:
\begin{align}
    \underset{\pi_{\theta}}{\max}\,\,\mathbb{E}_{x\sim\mathcal{D},y\sim\pi_{\theta}(y|x)}[r(x, y)] - \beta D_{\text{KL}}[\pi_{\theta}(y|x)\,\|\,\pi_{\text{ref}}(y|x)].\label{eqn:rlhf}    
\end{align}
\vspace{-\baselineskip}

Following~\cite{rafailov2023direct}, the problem above has a closed-form solution:
\vspace{-3pt}
\begin{align}
    \pi^*(y|x) = \frac{1}{Z(x)}\pi_{\text{ref}}(y|x)\exp\left(\frac{1}{\beta}r(x, y)\right),\label{eqn:optimal_policy}
\end{align}
\vspace{-\baselineskip}

where $Z(x) = \sum_y \pi_{\text{ref}}(y | x)\exp(\frac{1}{\beta}r(x, y))$ is the partition function. 

\begin{theorem}\label{thm:reward_optimization}
The reward optimization problem in \cref{eqn:rlhf} is equivalent to minimizing the KL-divergence $D_{\text{KL}}[\pi_{\theta}(y|x) \,\|\, \pi^*(y|x)]$, where $\pi^*(y|x)$ is the optimal policy defined in \cref{eqn:optimal_policy}.
\end{theorem}

All proofs are deferred to Appendix~\ref{app:proof}. LLM post-training normally learns a parameterized reward model $r_{\phi}$ or applies an implicit reward~\cite{rafailov2023direct} $r\propto \log[\pitheta / \piref]$ of the model $\pitheta$ as a proxy for human preference. However, in this paper, we propose a novel problem named \textit{unsupervised reward optimization for PLMs}. Our goal is to develop an \textit{explicit, a priori} score function $r(x, y)$ without \textit{any knowledge} of ground-truth labels for direct optimization on offline data. In the next section, we will present our unique design choices for such rewards along with sampling strategies for constructing the training sets.

\subsection{Task-Agnostic Reward Design and Data Sampling}\label{sec:reward_design}

The guiding principle of our unsupervised reward design is to find a proxy for \textit{controllability}, to measure how well a generated sample aligns with its prompt. 
Moreover, the reward function should be \textit{task-agnostic} for it to be truly unsupervised. Following current progress in URLVR~\cite{he2026how}, various intrinsic and extrinsic rewards induced by the reference policy $\pi_{\text{ref}}$\footnote{This enables offline training and self-improvement of the base model as the reference policy.} are examined. For the former, we investigate the sequence-level negative log-likelihood (predictive entropy) and its length-normalized variant (normalized entropy)
\vspace{-5pt}
\begin{align}
   \mathcal{H}_{\pi}(x, y) = -\sum_{i=1}^{|y|}\log\pi_{\text{ref}}(y_i|y_{<i}, x),
 \quad \tilde{\mathcal{H}}_{\pi}(x, y) = -\frac{1}{|y|}\sum_{i=1}^{|y|}\log\pi_{\text{ref}}(y_i|y_{<i}, x)\label{eqn:entropy} 
\end{align}
\vspace{-\baselineskip}

as measures of the model's uncertainty in the generated sequence $y$ given prompt $x$, which is widely adopted in entropy-based RL~\cite{cui2025entropy, ziebart2008maximum, haarnoja2017reinforcement, haarnoja2018soft}. We define the intrinsic reward $r_i(x, y)\triangleq -\mathcal{H}_{\pi}(x, y)$ or $-\tilde{\mathcal{H}}_{\pi}(x, y)$ to steer model towards more confident generations. For the latter, inspired by the notion of "semantic equivalence" in language tasks~\cite{kuhnsemantic, farquhar2024detecting}, we obtain protein embeddings using an off-the-shelf representation model $q(z|y)$ and evaluate the  distance between generated sequences
\vspace{-5pt}
\begin{align}
    \bar{d}_q (x, y) &= \mathbb{E}_{y'\sim\pi_{\text{ref}}(\cdot|x)} [d(q(y), q(y'))] 
    \simeq \frac{1}{N}\sum_{i=1}^N d(q(y), q(y^{(i)})), \quad \text{where} \,\,\,  y^{(i)} \sim \piref(\cdot \mid x),\label{eqn:semantic_dist}
\end{align}
\vspace{-\baselineskip}

$d$ can be any well-defined metric on the latent space, which provides a measure of semantic distance\footnote{For proteins, we hypothesize that this semantic distance serves as a faithful measure of the evolutionary, functional and structural similarity between sequences, given that the representation model $q$ is well-trained.} between generated samples. Intuitively, if a generated sequence is semantically distant from other generations conditioning on the same prompt, it is highly likely to be a statistical outlier or "hallucination" of the model. Therefore, we propose to define the extrinsic reward $r_e(x, y)\triangleq -\bar{d}_q (x, y)$ as the \textit{negative semantic distance} to encourage consistent generations. Our construction scales linearly with $\mathcal{O}(N)$ inference calls to model 
$q$, significantly improving upon prior LLM-based approaches~\cite{zhang2025right, kuhnsemantic, farquhar2024detecting} that rely on pairwise equivalence checks and exhibit complexity $\mathcal{O}(N^2)$.


Beyond designing proxy rewards, we further study temperature as a key factor in both offline data construction and reward labeling. Our rationale is two-fold. First, it's well-known that LLM generation statistics are highly sensitive to sampling temperatures~\cite{chen2021evaluating,zhu2024hot, zhang2024edt}, and in the same vein, PLMs could benefit from a multi-temperature sampling strategy to create more diverse offline dataset for unsupervised reward optimization (verified in~\Cref{app:ablation}). Second, we empirically observe that the proposed metrics in~\cref{eqn:entropy,eqn:semantic_dist} exhibit distinct correlation patterns with ground-truth labels at different temperatures. Consequently, for explicit reward optimization, the most strongly correlated metric should be selected at each temperature to minimize the noise of the reward signal. Therefore, given a prompt set $\mathcal{X}$ and temperature set $\mathcal{T}$, we define our multi-temperature sampling and reward labeling strategy for constructing offline training set $\mathcal{D}$ as follows:
\begin{align}
    \mathcal{D} = \bigcup_{x \in \mathcal{X}} \bigcup_{T \in \mathcal{T}} \left\{ \left(x, y^{(i)}, r_T(x, y^{(i)}) \right)\right\}_{i=1}^{N},
\quad \text{where} \,\,\, y^{(i)} \sim \piref(\cdot \mid x; T),\label{eqn:dataset}
\end{align}
\vspace{-\baselineskip}

$r_T$ is the best-performing reward at temperature $T$. We defer our implementation of $\mathcal{D}$ to~\Cref{sec:benchmark_rewards} since the optimal design choice is contingent on empirical evaluation. However, the final reward functions (one continuous and one binary) achieve consistent correlations with controllability metrics and deliver robust model improvements across conditional generation tasks (\Cref{tab:family-wise-stat}) and model scales (\Cref{tab:performance}), highlighting the generality and effectiveness of our task-agnostic design.

\subsection{Soft Reward Optimization (SRO)}\label{sec:SRO}
To see how a continuous reward can be maximized explicitly, substitute \cref{eqn:optimal_policy} into \Cref{thm:reward_optimization}, the KL objective $\mathcal{L}_{\text{KL}}$ becomes
\vspace{-8pt}
\begin{align}
    \underset{\pi_{\theta}}{\min}\, \mathcal{L}_{\text{KL}} &= 
    \underset{\pi_{\theta}}{\min}\,\mathbb{E}_{x\sim\mathcal{D}}\mathbb{E}_{y\sim\pi_{\theta}(y|x)}\left[\log\frac{\pi_{\theta}(y|x)}{\pi^*(y|x)}\right] \nonumber\\
    &\equiv  \underset{\pi_{\theta}}{\min}\,\mathbb{E}_{x\sim\mathcal{D}}\mathbb{E}_{y\sim\pi_{\text{ref}}(y|x)}\left[\frac{\pi_{\theta}}{\pi_{\text{ref}}}\left(\log\frac{\pi_{\theta}}{\pi_{\text{ref}}} - \frac{r(x, y)}{\beta}\right)\right]\label{eqn:IS}\\
    &\equiv  \underset{\pi_{\theta}}{\min}\,\mathbb{E}_{x\sim\mathcal{D}}\mathbb{E}_{y\sim\pi_{\text{ref}}(y|x)}\left[\frac{\exp(\frac{r_{\theta}(x, y)}{\beta})}{Z(x)}\left(\log\frac{\exp(\frac{r_{\theta}(x, y)}{\beta})}{Z(x)} - \log\frac{\exp(\frac{r(x, y)}{\beta})}{Z(x)}\right)\right]\label{eqn:reward_KL}
\end{align}
\vspace{-\baselineskip}

where the expectation of $\log Z(x)$ can be dropped in \cref{eqn:IS} or added in \cref{eqn:reward_KL} since it does not depend on $\pi_{\theta}$,  and $r_{\theta}(x, y)$ is the implicit reward parameterized by the model $\pi_{\theta}$~\cite{rafailov2023direct}
\begin{align}
    r_{\theta}(x, y) = \beta\log\frac{\pi_{\theta}(y|x)}{\pi_{\text{ref}}(y|x)} + \beta\log Z(x).\label{eqn:implicit_reward}
\end{align}
\vspace{-\baselineskip}

Since \Cref{eqn:reward_KL} can be interpreted as the KL-divergence between the distributions induced by the implicit reward $r_\theta(x, y)$ and explicit reward $r(x, y)$, prior methods~\cite{widatalla2024aligning,stocco2024guidinggenerativeproteinlanguage} approximate it using batch-normalized distributions $\exp(r_{\theta}(x, y^i))/\sum_i \exp(r_{\theta}(x, y^i))$ and  $\exp(r(x, y^i))/\sum_i \exp(r(x, y^i))$ up to a scaling factor $\beta$. \textit{To circumvent the batch effect}\footnote{In~\Cref{tab:hyperparameter}, we show that SRO variants substantially outperform this "weighted DPO" introduced by~\cite{widatalla2024aligning,stocco2024guidinggenerativeproteinlanguage}.}, we instead propose to directly estimate the RHS of \cref{eqn:IS} on $\mathcal{D}$ via importance-weighted Monte Carlo sampling:
\begin{align}
    \mathcal{L}_{\text{SRO}} &= \mathbb{E}_{x\sim\mathcal{D}}\mathbb{E}_{y\sim\pi_{\text{ref}}(y|x)}\left[\frac{\pi_{\theta}}{\pi_{\text{ref}}}\left(\log\frac{\pi_{\theta}}{\pi_{\text{ref}}} - \frac{r(x, y)}{\beta}-\alpha\right)\right]\nonumber\\
    &\simeq \frac{1}{|\mathcal{D}|
    }\sum_{i=1}^{|\mathcal{D}|}\frac{\pi_{\theta}(y^i|x^i)}{\pi_{\text{ref}}(y^i|x^i)}\left(\log\frac{\pi_{\theta}(y^i|x^i)}{\pi_{\text{ref}}(y^i|x^i)}-\frac{r(x^i, y^i)}{\beta}-\alpha\right)\label{eqn:SRO}
\end{align}
\vspace{-\baselineskip}

where $\alpha$ is a multiplier that constrains the overall probability mass of $\pitheta$ on $\mathcal{D}$. \Cref{eqn:SRO} operates as the objective of our proposed \textbf{soft reward optimization (SRO)}. 

\begin{theorem}\label{thm:SRO_optimality}
    Given a reward function $r(x, y)$, SRO in~\Cref{eqn:SRO} has a closed-form solution
    \begin{align}
        \pitheta^* = \pi_{\mathrm{ref}}\exp\left(\frac{r(x, y)}{\beta}+\alpha-1\right).
    \end{align}
\end{theorem}
\vspace{-\baselineskip}

For implementation, we sample a group of responses $\{y^j\}_{j=1}^G$ for each prompt $x^i$ with $\pi_{\text{ref}}$. We \textit{normalize}\footnote{In principle, this normalization is valid for algorithms like SRO which use global labels rather than groupwise or pairwise labels, \textit{only if} the overall and the promptwise reward statistics are consistent (e.g., exhibit similar high correlation with the ground-truth label), which is verified in~\Cref{sec:benchmark_rewards}.} the reward as an advantage function following RLVR methods like GRPO~\cite{shao2024deepseekmath}
\begin{align}
    \hat{r}(x, y) = \frac{r(x, y) - \text{mean}_y(r(x, y))}{\text{std}_y(r(x, y))}.\label{eqn:reward_norm}
\end{align}

\noindent By~\Cref{thm:SRO_optimality}, we choose $\alpha=1$ to ensure that $\pitheta^*(y|x)=\piref(y|x)$ for $y$ that receives zero (average) reward\footnote{Note that $\alpha=1$ does not strictly preserve the probability simplex constraint $\sum_y\pi_\theta(y|x)=1$ since we are making estimations on a finite, offline dataset. This choice empirically proves to be robust and effective in our experiments.}.


\subsection{Binarized Reward Optimization (BRO)}\label{sec:BRO}

In practice, as a proxy for ground-truth labels, the reward signal $r(x, y)$ can be noisy and inaccurate (see~\Cref{sec:benchmark_rewards}), which may render SRO unreliable. To this end, previous methods such as DPO~\cite{rafailov2023direct} circumvent the need for explicit reward by adopting the Bradley-Terry formalism~\cite{bradley1952rank} for preference distribution modeling. However, DPO and its variants require strictly ranked samples. As a more general alternative, we derive a novel optimization objective by leveraging \textit{unpaired} binary rewards introduced by KTO~\cite{ethayarajh2024model}. Specifically, we assume that instead of assigning a continuous score, the generated response $y$ can be classified as positive ($y\in\mathcal{D}^+$) or negative ($y\in\mathcal{D}^-$) in terms of alignment with its prompt $x$. This formulation, termed \textbf{binarized reward optimization (BRO)}, can be interpreted as an extreme case of $\mathcal{L}_{\text{SRO}}$ in the limit $\beta\to 0$:
\begin{align}
    \mathcal{L}_{\text{BRO}} &\triangleq \lim_{\beta\to 0}\mathcal{L}_{\text{SRO}}\nonumber\\
    &= \mathbb{E}_{x\sim\mathcal{D}}\mathbb{E}_{y\sim\pi_{\text{ref}}(y|x)}\left[\frac{\pi_{\theta}}{\pi_{\text{ref}}}\left(\log\frac{\pi_{\theta}}{\pi_{\text{ref}}} - \frac{r(x, y)}{\beta}\right)\right]\nonumber\\
    &= \mathbb{E}_{(x, y)\sim \mathcal{D}^+\cup \mathcal{D}^-}\left[\frac{\pi_{\theta}}{\pi_{\text{ref}}}\left(\log\frac{\pi_{\theta}}{\pi_{\text{ref}}} - r_{\text{BRO}}(x, y)\right)\right]
\end{align}
\vspace{-\baselineskip}

where
\begin{align}
r_{\text{BRO}}(x, y) =
\begin{cases}
  \infty,   & \text{if } \, y \in \mathcal{D}^+, \\
 -\infty,   & \text{if } \, y \in \mathcal{D}^-. \\
\end{cases}
\end{align}
\vspace{-\baselineskip}

To make $\mathcal{L}_{\text{BRO}}$ tractable, inspired by pre-existing methods~\cite{widatalla2024aligning,stocco2024guidinggenerativeproteinlanguage} that approximate the two probabilities in~\Cref{eqn:reward_KL} by their batch-normalized softmax variants, we propose optimizing their normalized variants in the binary label space:
\begin{align}
    \pi_{\theta, \text{BRO}} &\triangleq \sigma\left(\log\frac{\pi_{\theta}}{\pi_{\text{ref}}}\right) = \frac{\pi_{\theta}}{\pi_{\text{ref}} + \pi_{\theta}},\\
    \pi^*_{\text{BRO}} &\triangleq \sigma\left(r_{\text{BRO}}\right) = 
        \begin{cases}
          1,   & \text{if } y \in \mathcal{D}^+, \\
          0,   & \text{if } y \in \mathcal{D}^-. \\
\end{cases}
\end{align}
\vspace{-\baselineskip}

Optimizing $D_{\text{KL}}( \pi_{\theta, \text{BRO}}\,\|\,  \pi^*_{\text{BRO}})$ as a proxy for $D_{\text{KL}}(\pi_{\theta}\,\|\,\pi^*)$
requires constraining the probability mass of $\pitheta$ on $\mathcal{D}$ as in~\Cref{eqn:SRO}. Alternatively, we propose to directly minimize $D_{\text{KL}}(\pi^*_{\text{BRO}} \,\|\,  \pi_{\theta, \text{BRO}})$ instead:
\begin{align}
    \mathcal{L}_{\text{BRO}} &\simeq D_{KL}(\pi^*_{\text{BRO}} \,\|\,  \pi_{\theta, \text{BRO}})\nonumber\\
    &= H(\pi^*_{\text{BRO}}, \pi_{\theta, \text{BRO}}) - \cancel{H(\pi^*_{\text{BRO}})}\nonumber\\
    &= -\mathbb{E}_{(x, y)\sim\mathcal{D}}\left[\mathbbm{1}\{y\in\mathcal{D}^+\}\log\frac{\pi_{\theta}}{\pi_{\text{ref}} + \pi_{\theta}} + \left(1-\mathbbm{1}\{y\in\mathcal{D}^+\}\right)\log\frac{\pi_{\text{ref}}}{\pi_{\text{ref}} + \pi_{\theta}}\right]\label{eqn:obj_bro}
\end{align}
\vspace{-\baselineskip}

where $\mathbbm{1}(\cdot)$ is the indicator function, which is precisely the logistic regression loss for binary classification. From a game-theoretic view, the reward maximization problem induced by $\mathcal{L}_{\text{BRO}}$ in Eq.~\ref{eqn:obj_bro} can be interpreted as a zero-sum game where the learned policy $\pitheta$ competes against the reference policy $\piref$ over a labeled dataset $\mathcal{D} = \mathcal{D}^+ \cup \mathcal{D}^-$. For a detailed derivation, please see \Cref{app:game_BRO}.

\vspace{-5pt}
\section{Experiments}\label{sec:experiments}

We present an empirical study of the proposed \textit{unsupervised reward optimization} framework for PLMs through two main experimental components: (1) evaluating the proposed statistical metrics in~\Cref{sec:reward_design} for reward labeling, and (2) benchmarking PLMs fine-tuned by SRO/BRO on novel compositional OOD prompt generalization tasks with rewards from part (1). Our proposed methods show competitive performance despite being trained with noisy supervision, and in certain regimes achieve performance comparable to that of the oracle model fine-tuned with ground-truth labels.

\subsection{Experimental Setup}


Benchmarking the intrinsic and extrinsic rewards proposed in~\Cref{sec:reward_design} requires (1) a prompt set $\mathcal{X}$ as conditions, (2) a base model $\piref$ to sample protein sequences $\mathcal{Y}$, (3) a representation model $q$ to compute semantic distance, and (4) a ground-truth labeler $l: \mathcal{X} \times \mathcal{Y} \rightarrow \mathbb{R}$ for evaluating alignment with the prompt. To this end, we consider two frontier PLMs with customized setups:

\begin{itemize}[left=0pt, itemindent=1em, labelwidth=1em, labelsep=0.5em, align=left]

\item \textbf{ProGen2}~\cite{nijkamp2023progen2}: An autoregressive PLM pre-trained over one billion natural protein sequences and can be finetuned to generate highly specialized functional protein by instruction tuning. We use a 151M version called progen2-small-mix7 as the primary reference model, which is fine-tuned by~\citet{10821712} via SFT on a dataset of 7 Pfam~\cite{finn2014pfam} protein families. For a detailed description of these Pfam proteins, please see~\Cref{app:promptset}. The SFT prompts consist of special family tokens like \texttt{<|pf00002|>}. Following ESM3~\cite{hayes2025simulating}, we apply InterProScan~\cite{paysan2023interpro} as the ground-truth labeler by letting $l(x, y) = 1$ if the Pfam family keyword is recovered by InterProScan and $0$ otherwise. We also train a 764M progen2-medium-mix7 model via SFT on the same dataset for scaling experiments. 
\item \textbf{ESM3}~\cite{hayes2025simulating}: A state-of-the-art multimodal PLM with three tracks (sequence, structure and function) and pretrained via masked language modeling in~\Cref{eqn:pretrain_obj}. Due to its multimodal nature, ESM3 can perform diverse conditional generation tasks with customized functional and structural primitives. For Func2Seq tasks, we follow the same procedure as ProGen2 and use the keyword (InterPro entries) recovery rate by InterProScan as the ground-truth label. For Struct2Seq tasks, namely inverse folding~\cite{hsu2022learning, dauparas2022robust}, we predict the structure of generated sequence by ESMfold~\cite{lin2023evolutionary} and compute the backbone cRMSD against the prompted structure. The target structure is considered recovered if the $\text{cRMSD} < \SI{2}{\angstrom}$. The only open-weight model, 1.4B esm3-sm-open-v1, is used in this study.
\end{itemize}

\vspace{-6pt}
\textbf{Prompt and Dataset Construction} \,\,\, Since no prior work considers the exact setting of unsupervised reward optimization for PLMs, we introduce novel prompt sets for biologically meaningful Func2Seq and Struct2Seq tasks. For ESM3, we randomly sample 1000 proteins from the \textit{Dracunculus medinensis} (\textbf{DRAME})\footnote{We chose this dataset to prevent data leakage as \textit{Dracunculus medinensis} was added to AlphaFoldDB since v5 release, whereas ESM3 reported using v4 for training.} subset of AlphaFoldDB~\cite{varadi2022alphafold}, and employ their InterPro entries and 3D structures as the Func2Seq and Struct2Seq prompts respectively. The ESM3 prompt sets are used for benchmarking unsupervised rewards only. For ProGen2, based on the 7 Pfam control tags of progen2-small-mix7, we concatenate each family token with the first 10 residues of sequences from \textit{another} family to construct a novel OOD prompt set. For fine-tuning, the training and testing set comprise 1400 (\textbf{Pfam1400}) and 700 (\textbf{Pfam700}) prompts in total, with 200 and 100 entries from each Pfam family. These prompts are used for both reward benchmarking and BRO/SRO post-training to evaluate the \textit{compositional generalization}~\cite{kirk2023survey, yang2024exploring} of our proposed framework\footnote{Biologically, this setup can be interpreted to test the model's controllability as well as  creativity in "reprogramming" functional context from one protein family onto the N-terminal scaffold of another.}. Our full datasets $\mathcal{D}$ are sampled over five temperatures $T\in\{0.3, 0.5, 0.7, 1.0, 1.5\}$ according to~\Cref{eqn:dataset} with $N=64$, and therefore contain  $448k$ and $224k$ samples, respectively.  We apply a state-of-the-art representation PLM ESMC-600M~\cite{esm2024cambrian} for generating protein embeddings and extrinsic rewards. An alternative model is also assessed and produces consistent results, see details in~\Cref{app:reward_design}.




\vspace{-6pt}
\subsection{Benchmarking Unsupervised Rewards}\label{sec:benchmark_rewards}
\vspace{-4pt}

\begin{figure}[tbp]
    \centering
    \includegraphics[width=.94\linewidth]{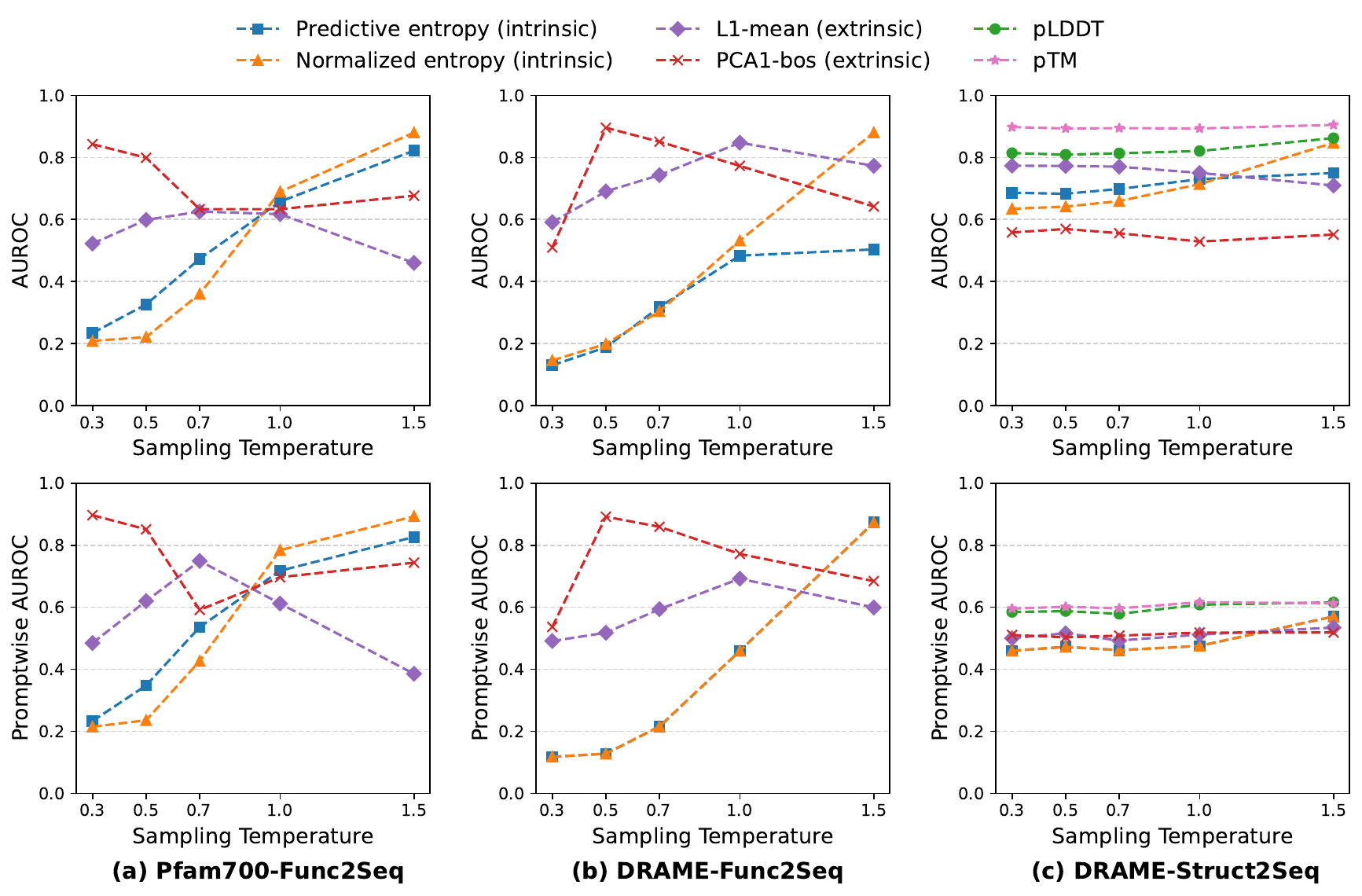}
    \vspace{-10pt}
    \caption{\textbf{Correlations between intrinsic/extrinsic rewards and the ground-truth label in terms of AUROC} on (a) Pfam700 prompt set for Func2Seq task, and DRAME prompt set for (b) Func2Seq and (c) Struct2Seq tasks, across five sampling temperatures. Promptwise AUROC evaluates the correlation within the generated sequences for each prompt.}
    \label{fig:reward_metrics}
    \vspace{-16pt}
\end{figure}

We examine our proposed unsupervised rewards in~\Cref{sec:reward_design} on three combinations of prompt sets and tasks: Pfam700-Func2Seq, DRAME-Func2Seq and DRAME-Struct2Seq. For extrinsic rewards defined in~\Cref{eqn:semantic_dist}, we evaluate a broad spectrum of metrics $d$ (see details in~\Cref{app:reward_design}) and select two representatives for illustration in~\Cref{fig:reward_metrics}: L1-mean and PCA1-bos.



For the Func2Seq tasks shown in~\Cref{fig:reward_metrics}(a) and (b), all metrics exhibit consistent trends across different PLMs and prompt sets. The entropy-based intrinsic rewards increase monotonically in temperature, starting from strongly negative correlation with the ground truth at $T\le 0.7$ and drastically reaching high positive correlation at $T> 1.0$. This suggests that both ProGen2 and ESM3 models are \textit{over-confident} at low temperatures, leading to degenerate and repetitive sequences with high probability~\cite{welleckneural, holtzmancurious} and \textit{under-confident} at high temperatures, generating incoherent and random outputs. Intriguingly, as observed in~\Cref{tab:post_training_result,app:critical_temp}, the "critical temperature" at which the models transition from being over-confident to under-confident, which is $T\sim 0.7$ for ProGen2 and $T\sim 1.0$ for ESM3, coincides with the best-performing temperature. However, at this critical temperature, intrinsic rewards are close to random, offering no predictive power.

Compared to predictive and normalized entropy, the extrinsic rewards exhibit more nuanced patterns. For PCA1-bos, the projection is up to a sign, which we always choose for it to be positively correlated with the ground truth\footnote{We empirically observe that at lower $T$, PCA1 effectively separates generated sequences into feasible-looking ones and degenerate gibberish, which aligns with a recent study~\cite{prabakaran2026quantifying}. Therefore, the sign of PCA can be determined by assigning the projection score of the non-biological sequences as negative, without any knowledge of ground-truth labels.}.
Therefore, the AUROCs for PCA1-bos are always above $0.5$, while generally decreasing in temperature. In contrast, L1-mean always peaks at the critical temperature, making it a \textit{great complement} to other metrics, especially the intrinsic uncertainty measures.

For Struct2Seq tasks in~\Cref{fig:reward_metrics}(c), we also consider pLDDT and pTM~\cite{jumper2021highly, lin2023evolutionary} from ESMfold prediction as competitive baselines. Compared to Func2Seq tasks, all proposed metrics demonstrate positive correlations with the ground truth and much more stable trends w.r.t. temperature. Despite being more computationally efficient\footnote{For our experiments with a single A40 GPU, the average inference time per sequence for ESMC-600M is less than $0.1$s, whereas ESMfold takes more than $2.0$s.}, they underperform relative to pLDDT and pTM. Moreover, unlike in Func2Seq tasks, the promptwise correlations for all metrics, including pLDDT and pTM, deteriorate markedly. We interpret this as evidence that correlations in Struct2Seq tasks are largely prompt-dependent: some structures are inherently easier to recover than others. This also implies that, for many prompts, all generated sequences are infeasible, as we empirically confirm, leaving few positive samples for self-improvement. As a result, we exclude Struct2Seq tasks from the subsequent post-training study, while retaining them as diagnostic evidence for our task-agnostic reward design.



Based on the analysis above, we instantiate the multi-temperature sampling and reward-labeling strategy in~\Cref{eqn:dataset} by defining a piecewise reward
function $r_T$ w.r.t. the critical temperature $T_c$: 
\begin{align}
    r_T = 
    \begin{cases}
    -\bar{d}_{\text{PCA1-bos}},   & \text{if } \, T < T_c \\
    -\bar{d}_{\text{L1-mean}},   & \text{if } \, T = T_c \\
    -\tilde{H}_{\pi},  & \text{if } \, T > T_c
    \end{cases}\label{eqn:SRO_reward}
\end{align}
\vspace{-\baselineskip}


which is motivated by the diagnostic correlations in~\Cref{fig:reward_metrics}. See~\Cref{app:ablation} for the ablation study of design choices. For each prompt, we sample $N=64$ sequences and retain only the top-$4$ and bottom-$4$ for post-training to reduce reward noise. Moreover, this construction naturally produces positive and negative samples for BRO/DPO/KTO. More detail can be found in~\Cref{app:model_training}.

\vspace{-5pt}
\subsection{Compositional OOD Prompt Generalization}

To demonstrate the effectiveness of our proposed framework, we fine-tune two SFT models, Progen2-small-mix7 and Progen2-medium-mix7, via SRO/BRO (ours) and two competitive baselines (DPO~\cite{rafailov2023direct} and KTO~\cite{ethayarajh2024model}) with the unsupervised rewards in~\Cref{eqn:SRO_reward}. For binary rewards used by BRO, DPO and KTO, we find that adopting L1-mean alone regardless of $T$ produces robust improvements for all post-training methods (\Cref{tab:ablation_reward}), which we report here. As oracle baselines, we also train base models with ground-truth reward via BRO, since the labels are inherently binary.

\begin{table}[!tbhp]
  \centering
  \caption{\textbf{Keyword recovery on Pfam700 Func2Seq tasks}. We report pass@1 success rate with top-$k$ decoding, where $k=15$. For each temperature, the
  best model except Oracle is \textbf{bolded} and the second best is \underline{underlined}. More details can be found in~\Cref{app:model_training}.}
  \label{tab:performance}
  \resizebox{\linewidth}{!}{%
  \begin{tabular}{lcccccc}
  \toprule
   & {$T=0.3$} & {$T=0.5$} & {$T=0.7$} & {$T=1.0$} & {$T=1.5$} & Average \\
  \midrule
  \multicolumn{7}{l}{\textbf{151M model}} \\
  \midrule
  Progen2-small-mix7 & 0.210 & 0.288 & 0.390 & 0.401 & 0.118 & 0.281 \\
  Progen2-small-mix7 DPO & 0.283 & 0.395 & \textbf{0.512} & 0.435 & 0.122 & 0.350\\
  Progen2-small-mix7 KTO & 0.272 & 0.347 & 0.430 & 0.458 & 0.175 & 0.337 \\
  \rowcolor{urorow}
  Progen2-small-mix7 BRO & \textbf{0.406} & \underline{0.445} & 0.494 & \textbf{0.520} & \underline{0.298} & \textbf{0.433} \\
  \rowcolor{urorow}
  Progen2-small-mix7 SRO & \underline{0.383}  & \textbf{0.455} & \underline{0.498} & \underline{0.499} &  \textbf{0.301} &  \underline{0.427} \\
  Progen2-small-mix7 Oracle & 0.537 & 0.573 & 0.579 & 0.506 & 0.247 & 0.488 \\
  \midrule
  \multicolumn{7}{l}{\textbf{764M model}} \\
  \midrule
  Progen2-medium-mix7 & 0.383 & 0.471 & 0.567 & 0.428 & 0.074 & 0.385 \\
  Progen2-medium-mix7 DPO & 0.591 & 0.693 & 0.670 & 0.379 & 0.067 & 0.480 \\
  Progen2-medium-mix7 KTO &  0.507 & 0.620 & 0.681 & 0.557 & 0.159 &  0.505\\
  \rowcolor{urorow}
  Progen2-medium-mix7 BRO & \textbf{0.678} & \underline{0.697} & \underline{0.701} & \underline{0.645} & \textbf{0.342} & \textbf{0.613} \\
  \rowcolor{urorow}
  Progen2-medium-mix7 SRO & \underline{0.655} & \textbf{0.706} & \textbf{0.708} & \textbf{0.648} & \underline{0.292} & \underline{0.602} \\
  Progen2-medium-mix7 Oracle & 0.661 & 0.700 & 0.704 & 0.642 & 0.308 & 0.603 \\
  \bottomrule
  \end{tabular}%
  }
  \label{tab:post_training_result}
  \end{table}

Shown in~\Cref{tab:post_training_result,tab:family-wise-stat}, both SRO and BRO significantly improve the generalization of the base model on the challenging compositional OOD prompt set \textbf{Pfam700} across temperatures, model scales and protein families.
 They outperform DPO and KTO by a substantial margin and, in some regimes, achieve performance comparable to the Oracle. For example, for the 764M base model, BRO attains
  an average recovery rate of $61.3\%$, compared with $60.3\%$ for the Oracle. Moreover,  we conduct pass@k analysis~\cite{chen2021evaluating} in~\Cref{fig:pass@k} to further dissect the nature of the improvements conferred by BRO and SRO. We observe that they generally boost pass@k scores for small $k$ yet converge with the base model at high $k$ values. This implies that BRO/SRO incentivize the model to prioritize and consistently generate patterns that align with the prompts, rather than instilling fundamentally novel ones, reconfirming recent studies on post-training reasoning LLMs~\cite{zhang2025right, wu2024reft, yuedoes,  songmind}. However, DPO, despite underperforming BRO/SRO at small $k$, exhibits stronger performance compared to the base model for nearly all $k$ values, suggesting that it not only enhances the probability of successful generations, but also pushes the boundary of the model's compositional generalization capabilities. Overall, these observations highlight the effectiveness and potential of our refined strategy for reward labeling, data sampling and algorithm design. 
\begin{figure}[bh]
    \centering
    \includegraphics[width=\linewidth]{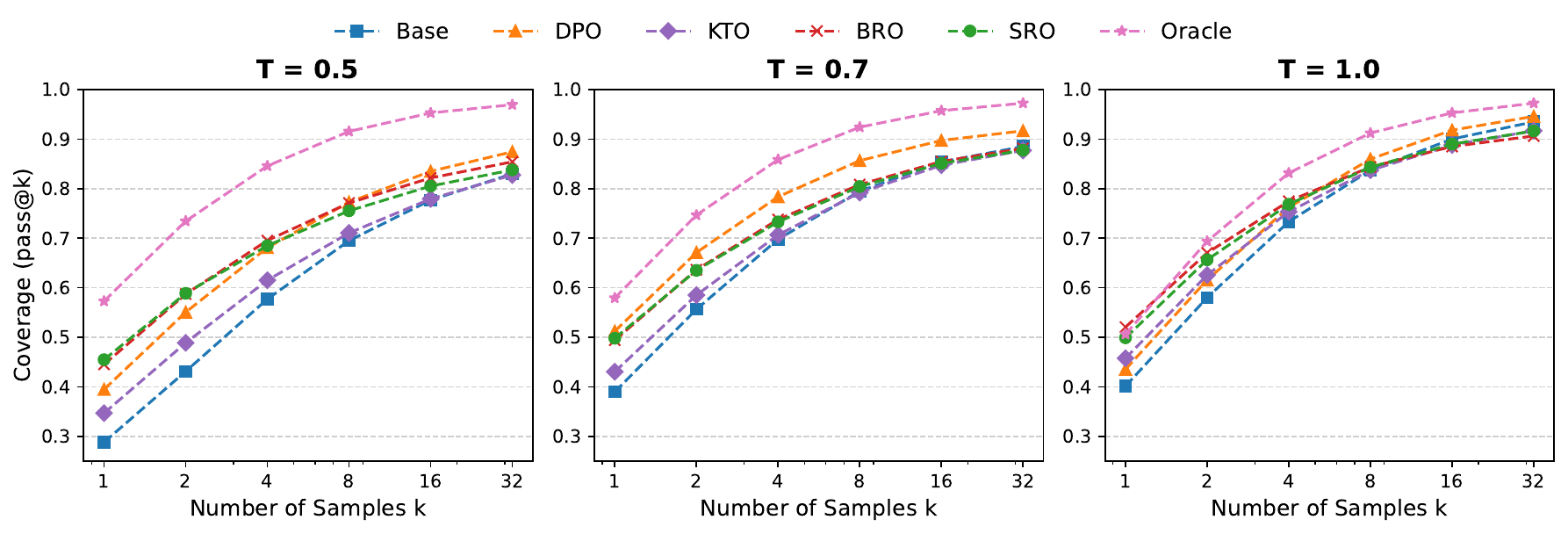}
    \caption{\textbf{Pass@k curves for Progen2-small-mix7 on Pfam700 Func2Seq tasks.} For each prompt, $N=64$ sequences are sampled for evaluation.}
    \label{fig:pass@k}
    \vspace{-\baselineskip}
\end{figure}


\bibliography{reference}

@article{vaswani2017attention,
  title={Attention is all you need},
  author={Vaswani, Ashish and Shazeer, Noam and Parmar, Niki and Uszkoreit, Jakob and Jones, Llion and Gomez, Aidan N and Kaiser, {\L}ukasz and Polosukhin, Illia},
  journal={Advances in neural information processing systems},
  volume={30},
  year={2017}
}

@inproceedings{devlin2019bert,
  title={Bert: Pre-training of deep bidirectional transformers for language understanding},
  author={Devlin, Jacob and Chang, Ming-Wei and Lee, Kenton and Toutanova, Kristina},
  booktitle={Proceedings of the 2019 conference of the North American chapter of the association for computational linguistics: human language technologies, volume 1 (long and short papers)},
  pages={4171--4186},
  year={2019}
}

@article{finn2014pfam,
  title={Pfam: the protein families database},
  author={Finn, Robert D and Bateman, Alex and Clements, Jody and Coggill, Penelope and Eberhardt, Ruth Y and Eddy, Sean R and Heger, Andreas and Hetherington, Kirstie and Holm, Liisa and Mistry, Jaina and others},
  journal={Nucleic acids research},
  volume={42},
  number={D1},
  pages={D222--D230},
  year={2014},
  publisher={Oxford University Press}
}

@article{varadi2022alphafold,
  title={AlphaFold Protein Structure Database: massively expanding the structural coverage of protein-sequence space with high-accuracy models},
  author={Varadi, Mihaly and Anyango, Stephen and Deshpande, Mandar and Nair, Sreenath and Natassia, Cindy and Yordanova, Galabina and Yuan, David and Stroe, Oana and Wood, Gemma and Laydon, Agata and others},
  journal={Nucleic acids research},
  volume={50},
  number={D1},
  pages={D439--D444},
  year={2022},
  publisher={Oxford University Press}
}

@article{paysan2023interpro,
  title={InterPro in 2022},
  author={Paysan-Lafosse, Typhaine and Blum, Matthias and Chuguransky, Sara and Grego, Tiago and Pinto, Beatriz L{\'a}zaro and Salazar, Gustavo A and Bileschi, Maxwell L and Bork, Peer and Bridge, Alan and Colwell, Lucy and others},
  journal={Nucleic acids research},
  volume={51},
  number={D1},
  pages={D418--D427},
  year={2023},
  publisher={Oxford University Press}
}

@article{jumper2021highly,
  title={Highly accurate protein structure prediction with AlphaFold},
  author={Jumper, John and Evans, Richard and Pritzel, Alexander and Green, Tim and Figurnov, Michael and Ronneberger, Olaf and Tunyasuvunakool, Kathryn and Bates, Russ and {\v{Z}}{\'\i}dek, Augustin and Potapenko, Anna and others},
  journal={nature},
  volume={596},
  number={7873},
  pages={583--589},
  year={2021},
  publisher={Nature Publishing Group UK London}
}

@article{baek2021accurate,
  title={Accurate prediction of protein structures and interactions using a three-track neural network},
  author={Baek, Minkyung and DiMaio, Frank and Anishchenko, Ivan and Dauparas, Justas and Ovchinnikov, Sergey and Lee, Gyu Rie and Wang, Jue and Cong, Qian and Kinch, Lisa N and Schaeffer, R Dustin and others},
  journal={Science},
  volume={373},
  number={6557},
  pages={871--876},
  year={2021},
  publisher={American Association for the Advancement of Science}
}

@article{lin2023evolutionary,
  title={Evolutionary-scale prediction of atomic-level protein structure with a language model},
  author={Lin, Zeming and Akin, Halil and Rao, Roshan and Hie, Brian and Zhu, Zhongkai and Lu, Wenting and Smetanin, Nikita and Verkuil, Robert and Kabeli, Ori and Shmueli, Yaniv and others},
  journal={Science},
  volume={379},
  number={6637},
  pages={1123--1130},
  year={2023},
  publisher={American Association for the Advancement of Science}
}

@article{kortemme2024novo,
  title={De novo protein design—From new structures to programmable functions},
  author={Kortemme, Tanja},
  journal={Cell},
  volume={187},
  number={3},
  pages={526--544},
  year={2024},
  publisher={Elsevier}
}

@article{winnifrith2024generative,
  title={Generative artificial intelligence for de novo protein design},
  author={Winnifrith, Adam and Outeiral, Carlos and Hie, Brian L},
  journal={Current Opinion in Structural Biology},
  volume={86},
  pages={102794},
  year={2024},
  publisher={Elsevier}
}

@article{yang2026past,
  title={The past, present and future of de novo protein design},
  author={Yang, Wei and Wang, Shunzhi and Lee, Gyu Rie and Zhang, Jason Z and Courbet, Alexis and Juergens, David and Wang, Xinru and Schlichthaerle, Thomas and Abedi, Mohamad and Ragotte, Robert and others},
  journal={Nature},
  volume={652},
  number={8112},
  pages={1139--1152},
  year={2026},
  publisher={Nature Publishing Group UK London}
}

@article{mo2026sadit,
  title={SaDiT: Efficient Protein Backbone Design via Latent Structural Tokenization and Diffusion Transformers},
  author={Mo, Shentong and Li, Lanqing},
  journal={arXiv preprint arXiv:2602.06706},
  year={2026}
}

@article{dauparas2022robust,
  title={Robust deep learning--based protein sequence design using ProteinMPNN},
  author={Dauparas, Justas and Anishchenko, Ivan and Bennett, Nathaniel and Bai, Hua and Ragotte, Robert J and Milles, Lukas F and Wicky, Basile IM and Courbet, Alexis and de Haas, Rob J and Bethel, Neville and others},
  journal={Science},
  volume={378},
  number={6615},
  pages={49--56},
  year={2022},
  publisher={American Association for the Advancement of Science}
}

@inproceedings{hsu2022learning,
  title={Learning inverse folding from millions of predicted structures},
  author={Hsu, Chloe and Verkuil, Robert and Liu, Jason and Lin, Zeming and Hie, Brian and Sercu, Tom and Lerer, Adam and Rives, Alexander},
  booktitle={International conference on machine learning},
  pages={8946--8970},
  year={2022},
  organization={PMLR}
}

@article{hayes2025simulating,
  title={Simulating 500 million years of evolution with a language model},
  author={Hayes, Thomas and Rao, Roshan and Akin, Halil and Sofroniew, Nicholas J and Oktay, Deniz and Lin, Zeming and Verkuil, Robert and Tran, Vincent Q and Deaton, Jonathan and Wiggert, Marius and others},
  journal={Science},
  volume={387},
  number={6736},
  pages={850--858},
  year={2025},
  publisher={American Association for the Advancement of Science}
}

@article{madani2023large,
  title={Large language models generate functional protein sequences across diverse families},
  author={Madani, Ali and Krause, Ben and Greene, Eric R and Subramanian, Subu and Mohr, Benjamin P and Holton, James M and Olmos Jr, Jose Luis and Xiong, Caiming and Sun, Zachary Z and Socher, Richard and others},
  journal={Nature biotechnology},
  volume={41},
  number={8},
  pages={1099--1106},
  year={2023},
  publisher={Nature Publishing Group US New York}
}

@article{nijkamp2023progen2,
  title={Progen2: exploring the boundaries of protein language models},
  author={Nijkamp, Erik and Ruffolo, Jeffrey A and Weinstein, Eli N and Naik, Nikhil and Madani, Ali},
  journal={Cell systems},
  volume={14},
  number={11},
  pages={968--978},
  year={2023},
  publisher={Elsevier}
}

@ARTICLE{9477085,
  author={Elnaggar, Ahmed and Heinzinger, Michael and Dallago, Christian and Rehawi, Ghalia and Wang, Yu and Jones, Llion and Gibbs, Tom and Feher, Tamas and Angerer, Christoph and Steinegger, Martin and Bhowmik, Debsindhu and Rost, Burkhard},
  journal={IEEE Transactions on Pattern Analysis and Machine Intelligence}, 
  title={ProtTrans: Toward Understanding the Language of Life Through Self-Supervised Learning}, 
  year={2022},
  volume={44},
  number={10},
  pages={7112-7127},
  doi={10.1109/TPAMI.2021.3095381}}

@INPROCEEDINGS{10821712,
  author={Hrbáň, Hugo and Hoksza, David},
  booktitle={2024 IEEE International Conference on Bioinformatics and Biomedicine (BIBM)}, 
  title={Protein Family Sequence Generation through ProGen2 Fine-Tuning}, 
  year={2024},
  volume={},
  number={},
  pages={7037-7039},
  doi={10.1109/BIBM62325.2024.10821712}}

@article{rives2019biological,
  author={Rives, Alexander and Meier, Joshua and Sercu, Tom and Goyal, Siddharth and Lin, Zeming and Liu, Jason and Guo, Demi and Ott, Myle and Zitnick, C. Lawrence and Ma, Jerry and Fergus, Rob},
  title={Biological Structure and Function Emerge from Scaling Unsupervised Learning to 250 Million Protein Sequences},
  year={2019},
  doi={10.1101/622803},
  url={https://www.biorxiv.org/content/10.1101/622803v4},
  journal={PNAS}
}

@article{watson2023novo,
  title={De novo design of protein structure and function with RFdiffusion},
  author={Watson, Joseph L and Juergens, David and Bennett, Nathaniel R and Trippe, Brian L and Yim, Jason and Eisenach, Helen E and Ahern, Woody and Borst, Andrew J and Ragotte, Robert J and Milles, Lukas F and others},
  journal={Nature},
  volume={620},
  number={7976},
  pages={1089--1100},
  year={2023},
  publisher={Nature Publishing Group UK London}
}

@article{alley2019unified,
  title={Unified rational protein engineering with sequence-based deep representation learning},
  author={Alley, Ethan C and Khimulya, Grigory and Biswas, Surojit and AlQuraishi, Mohammed and Church, George M},
  journal={Nature methods},
  volume={16},
  number={12},
  pages={1315--1322},
  year={2019},
  publisher={Nature Publishing Group US New York}
}

@article{chen2022interpretable,
  title={Interpretable RNA foundation model from unannotated data for highly accurate RNA structure and function predictions},
  author={Chen, Jiayang and Hu, Zhihang and Sun, Siqi and Tan, Qingxiong and Wang, Yixuan and Yu, Qinze and Zong, Licheng and Hong, Liang and Xiao, Jin and Shen, Tao and others},
  journal={arXiv preprint arXiv:2204.00300},
  year={2022}
}

@inproceedings{susaprot2024,
  title={SaProt: Protein Language Modeling with Structure-aware Vocabulary},
  author={Su, Jin and Han, Chenchen and Zhou, Yuyang and Shan, Junjie and Zhou, Xibin and Yuan, Fajie},
  booktitle={The Twelfth International Conference on Learning Representations},
  year={2024}
}

@article{nguyen2024sequence,
  title={Sequence modeling and design from molecular to genome scale with Evo},
  author={Nguyen, Eric and Poli, Michael and Durrant, Matthew G and Kang, Brian and Katrekar, Dhruva and Li, David B and Bartie, Liam J and Thomas, Armin W and King, Samuel H and Brixi, Garyk and others},
  journal={Science},
  volume={386},
  number={6723},
  pages={eado9336},
  year={2024},
  publisher={American Association for the Advancement of Science}
}

@article{meier2021language,
  title={Language models enable zero-shot prediction of the effects of mutations on protein function},
  author={Meier, Joshua and Rao, Roshan and Verkuil, Robert and Liu, Jason and Sercu, Tom and Rives, Alex},
  journal={Advances in neural information processing systems},
  volume={34},
  pages={29287--29303},
  year={2021}
}

@article{heinzinger2019modeling,
  title={Modeling aspects of the language of life through transfer-learning protein sequences},
  author={Heinzinger, Michael and Elnaggar, Ahmed and Wang, Yu and Dallago, Christian and Nechaev, Dmitrii and Matthes, Florian and Rost, Burkhard},
  journal={BMC bioinformatics},
  volume={20},
  number={1},
  pages={723},
  year={2019},
  publisher={Springer}
}

@inproceedings{raotransformer2021,
  title={Transformer protein language models are unsupervised structure learners},
  author={Rao, Roshan and Meier, Joshua and Sercu, Tom and Ovchinnikov, Sergey and Rives, Alexander},
  booktitle={International Conference on Learning Representations},
  year={2021}
}

@article{lewis2025scalable,
  title={Scalable emulation of protein equilibrium ensembles with generative deep learning},
  author={Lewis, Sarah and Hempel, Tim and Jim{\'e}nez-Luna, Jos{\'e} and Gastegger, Michael and Xie, Yu and Foong, Andrew YK and Satorras, Victor Garc{\'\i}a and Abdin, Osama and Veeling, Bastiaan S and Zaporozhets, Iryna and others},
  journal={Science},
  volume={389},
  number={6761},
  pages={eadv9817},
  year={2025},
  publisher={American Association for the Advancement of Science}
}

@article{passaro2025boltz,
  title={Boltz-2: Towards accurate and efficient binding affinity prediction},
  author={Passaro, Saro and Corso, Gabriele and Wohlwend, Jeremy and Reveiz, Mateo and Thaler, Stephan and Somnath, Vignesh Ram and Getz, Noah and Portnoi, Tally and Roy, Julien and Stark, Hannes and others},
  journal={BioRxiv},
  year={2025}
}

@article{chen2025xtrimopglm,
  title={xTrimoPGLM: unified 100-billion-parameter pretrained transformer for deciphering the language of proteins},
  author={Chen, Bo and Cheng, Xingyi and Li, Pan and Geng, Yangli-ao and Gong, Jing and Li, Shen and Bei, Zhilei and Tan, Xu and Wang, Boyan and Zeng, Xin and others},
  journal={Nature Methods},
  volume={22},
  number={5},
  pages={1028--1039},
  year={2025},
  publisher={Nature Publishing Group US New York}
}

@article{lv2025prollama,
  title={Prollama: A protein large language model for multi-task protein language processing},
  author={Lv, Liuzhenghao and Lin, Zongying and Li, Hao and Liu, Yuyang and Cui, Jiaxi and Chen, Calvin Yu-Chian and Yuan, Li and Tian, Yonghong},
  journal={IEEE Transactions on Artificial Intelligence},
  year={2025},
  publisher={IEEE}
}

@misc{stocco2024guidinggenerativeproteinlanguage,
      title={Guiding Generative Protein Language Models with Reinforcement Learning}, 
      author={Filippo Stocco and Maria Artigues-Lleixa and Andrea Hunklinger and Talal Widatalla and Marc Guell and Noelia Ferruz},
      year={2024},
      eprint={2412.12979},
      archivePrefix={arXiv},
      primaryClass={q-bio.BM},
      url={https://arxiv.org/abs/2412.12979}, 
}

@article{widatalla2024aligning,
  title={Aligning protein generative models with experimental fitness via direct preference optimization},
  author={Widatalla, Talal and Rafailov, Rafael and Hie, Brian},
  journal={bioRxiv},
  pages={2024--05},
  year={2024},
  publisher={Cold Spring Harbor Laboratory}
}

@article{jiang2024rapid,
  title={Rapid in silico directed evolution by a protein language model with EVOLVEpro},
  author={Jiang, Kaiyi and Yan, Zhaoqing and Di Bernardo, Matteo and Sgrizzi, Samantha R and Villiger, Lukas and Kayabolen, Alisan and Kim, BJ and Carscadden, Josephine K and Hiraizumi, Masahiro and Nishimasu, Hiroshi and others},
  journal={Science},
  volume={387},
  number={6732},
  pages={eadr6006},
  year={2024},
  publisher={American Association for the Advancement of Science}
}

@article{yang2019machine,
  title={Machine-learning-guided directed evolution for protein engineering},
  author={Yang, Kevin K and Wu, Zachary and Arnold, Frances H},
  journal={Nature methods},
  volume={16},
  number={8},
  pages={687--694},
  year={2019},
  publisher={Nature Publishing Group US New York}
}

@article{bai2022training,
  title={Training a helpful and harmless assistant with reinforcement learning from human feedback},
  author={Bai, Yuntao and Jones, Andy and Ndousse, Kamal and Askell, Amanda and Chen, Anna and DasSarma, Nova and Drain, Dawn and Fort, Stanislav and Ganguli, Deep and Henighan, Tom and others},
  journal={arXiv preprint arXiv:2204.05862},
  year={2022}
}

@inproceedings{korbak2023pretraining,
  title={Pretraining language models with human preferences},
  author={Korbak, Tomasz and Shi, Kejian and Chen, Angelica and Bhalerao, Rasika Vinayak and Buckley, Christopher and Phang, Jason and Bowman, Samuel R and Perez, Ethan},
  booktitle={International conference on machine learning},
  pages={17506--17533},
  year={2023},
  organization={PMLR}
}

@article{stiennon2020learning,
  title={Learning to summarize with human feedback},
  author={Stiennon, Nisan and Ouyang, Long and Wu, Jeffrey and Ziegler, Daniel and Lowe, Ryan and Voss, Chelsea and Radford, Alec and Amodei, Dario and Christiano, Paul F},
  journal={Advances in neural information processing systems},
  volume={33},
  pages={3008--3021},
  year={2020}
}

@article{ouyang2022training,
  title={Training language models to follow instructions with human feedback},
  author={Ouyang, Long and Wu, Jeffrey and Jiang, Xu and Almeida, Diogo and Wainwright, Carroll and Mishkin, Pamela and Zhang, Chong and Agarwal, Sandhini and Slama, Katarina and Ray, Alex and others},
  journal={Advances in neural information processing systems},
  volume={35},
  pages={27730--27744},
  year={2022}
}

@article{ziegler2019fine,
  title={Fine-tuning language models from human preferences},
  author={Ziegler, Daniel M and Stiennon, Nisan and Wu, Jeffrey and Brown, Tom B and Radford, Alec and Amodei, Dario and Christiano, Paul and Irving, Geoffrey},
  journal={arXiv preprint arXiv:1909.08593},
  year={2019}
}

@inproceedings{yuedoes,
  title={Does Reinforcement Learning Really Incentivize Reasoning Capacity in LLMs Beyond the Base Model?},
  author={Yue, Yang and Chen, Zhiqi and Lu, Rui and Zhao, Andrew and Wang, Zhaokai and Song, Shiji and Huang, Gao},
  booktitle={The Thirty-ninth Annual Conference on Neural Information Processing Systems},
  year={2025}
}

@article{rafailov2023direct,
  title={Direct preference optimization: Your language model is secretly a reward model},
  author={Rafailov, Rafael and Sharma, Archit and Mitchell, Eric and Manning, Christopher D and Ermon, Stefano and Finn, Chelsea},
  journal={Advances in neural information processing systems},
  volume={36},
  pages={53728--53741},
  year={2023}
}

@inproceedings{ethayarajh2024model,
  title={Model Alignment as Prospect Theoretic Optimization},
  author={Ethayarajh, Kawin and Xu, Winnie and Muennighoff, Niklas and Jurafsky, Dan and Kiela, Douwe},
  booktitle={International Conference on Machine Learning},
  pages={12634--12651},
  year={2024},
  organization={PMLR}
}

@inproceedings{weifinetuned,
  title={Finetuned Language Models are Zero-Shot Learners},
  author={Wei, Jason and Bosma, Maarten and Zhao, Vincent and Guu, Kelvin and Yu, Adams Wei and Lester, Brian and Du, Nan and Dai, Andrew M and Le, Quoc V},
  booktitle={International Conference on Learning Representations},
  year={2022}
}

@inproceedings{huang2025steering,
  title={Steering Protein Language Models},
  author={Huang, Long-Kai and Zhu, Rongyi and He, Bing and Yao, Jianhua},
  booktitle={International Conference on Machine Learning},
  pages={26247--26260},
  year={2025},
  organization={PMLR}
}

@article{comanici2025gemini,
  title={Gemini 2.5: Pushing the frontier with advanced reasoning, multimodality, long context, and next generation agentic capabilities},
  author={Comanici, Gheorghe and Bieber, Eric and Schaekermann, Mike and Pasupat, Ice and Sachdeva, Noveen and Dhillon, Inderjit and Blistein, Marcel and Ram, Ori and Zhang, Dan and Rosen, Evan and others},
  journal={arXiv preprint arXiv:2507.06261},
  year={2025}
}

@misc{team2025qwq,
  title={Qwq-32b: Embracing the power of reinforcement learning},
  author={Team, Qwen},
  year={2025},
  publisher={March},
  url={https://qwenlm.github.io/blog/qwq-32b/}
}

@article{yang2025qwen3,
  title={Qwen3 technical report},
  author={Yang, An and Li, Anfeng and Yang, Baosong and Zhang, Beichen and Hui, Binyuan and Zheng, Bo and Yu, Bowen and Gao, Chang and Huang, Chengen and Lv, Chenxu and others},
  journal={arXiv preprint arXiv:2505.09388},
  year={2025}
}

@article{bradley1952rank,
  title={Rank analysis of incomplete block designs: I. the method of paired comparisons},
  author={Bradley, Ralph Allan and Terry, Milton E},
  journal={Biometrika},
  volume={39},
  number={3/4},
  pages={324--345},
  year={1952},
  publisher={JSTOR}
}

@inproceedings{
he2026how,
title={How Far Can Unsupervised {RLVR} Scale {LLM} Training?},
author={Bingxiang He and Yuxin Zuo and Zeyuan Liu and Shangziqi Zhao and Zixuan Fu and Junlin Yang and Cheng Qian and Kaiyan Zhang and Yuchen Fan and Ganqu Cui and Xiusi Chen and Youbang Sun and Xingtai Lv and Xuekai Zhu and Li Sheng and Ran Li and Huan-ang Gao and Yuchen Zhang and Lifan Yuan and Bowen Zhou and others},
booktitle={The Thirteenth International Conference on Learning Representations},
year={2026}
}

@inproceedings{
agarwal2025unreasonable,
title={The Unreasonable Effectiveness of Entropy Minimization in {LLM} Reasoning},
author={Shivam Agarwal and Zimin Zhang and Lifan Yuan and Jiawei Han and Hao Peng},
booktitle={Thirty-ninth Conference on Neural Information Processing Systems},
year={2025},
url={https://openreview.net/forum?id=UfFTBEsLgI},
note={Poster Presentation}
}

@inproceedings{
zuo2025ttrl,
title={{TTRL}: Test-Time Reinforcement Learning},
author={Yuxin Zuo and Kaiyan Zhang and Li Sheng and Shang Qu and Ganqu Cui and Xuekai Zhu and Haozhan Li and Yuchen Zhang and Xinwei Long and Ermo Hua and Biqing Qi and Youbang Sun and Zhiyuan Ma and Lifan Yuan and Ning Ding and Bowen Zhou},
booktitle={Thirty-ninth Conference on Neural Information Processing Systems},
year={2025},
url={https://openreview.net/forum?id=VuVhgEiu20},
note={Poster Presentation}
}

@inproceedings{
zhang2025right,
title={Right Question is Already Half the Answer: Fully Unsupervised LLM Reasoning Incentivization},
author={Qingyang Zhang and Haitao Wu and Changqing Zhang and Peilin Zhao and Yatao Bian},
booktitle={Thirty-ninth Conference on Neural Information Processing Systems},
year={2025},
url={https://openreview.net/forum?id=k8Mim6RI5O},
note={Spotlight Presentation}
}

@inproceedings{zhaoabsolute,
  title={Absolute Zero: Reinforced Self-play Reasoning with Zero Data},
  author={Zhao, Andrew and Wu, Yiran and Yue, Yang and Wu, Tong and Xu, Quentin and Lin, Matthieu and Wang, Shenzhi and Wu, Qingyun and Zheng, Zilong and Huang, Gao},
  booktitle={The Thirty-ninth Annual Conference on Neural Information Processing Systems}
}

@article{cui2025entropy,
  title={The entropy mechanism of reinforcement learning for reasoning language models},
  author={Cui, Ganqu and Zhang, Yuchen and Chen, Jiacheng and Yuan, Lifan and Wang, Zhi and Zuo, Yuxin and Li, Haozhan and Fan, Yuchen and Chen, Huayu and Chen, Weize and others},
  journal={arXiv preprint arXiv:2505.22617},
  year={2025}
}

@article{zhang2025no,
  title={No free lunch: Rethinking internal feedback for llm reasoning},
  author={Zhang, Yanzhi and Zhang, Zhaoxi and Guan, Haoxiang and Cheng, Yilin and Duan, Yitong and Wang, Chen and Wang, Yue and Zheng, Shuxin and He, Jiyan},
  journal={arXiv preprint arXiv:2506.17219},
  year={2025}
}

@inproceedings{burnsweak,
  title={Weak-to-Strong Generalization: Eliciting Strong Capabilities With Weak Supervision},
  author={Burns, Collin and Izmailov, Pavel and Kirchner, Jan Hendrik and Baker, Bowen and Gao, Leo and Aschenbrenner, Leopold and Chen, Yining and Ecoffet, Adrien and Joglekar, Manas and Leike, Jan and others},
  booktitle={Forty-first International Conference on Machine Learning}
}

@article{silver2025welcome,
  title={Welcome to the era of experience},
  author={Silver, David and Sutton, Richard S},
  journal={Google AI},
  volume={1},
  pages={11},
  year={2025}
}

@article{shafayat2025can,
  title={Can Large Reasoning Models Self-Train?},
  author={Shafayat, Sheikh and Tajwar, Fahim and Salakhutdinov, Ruslan and Schneider, Jeff and Zanette, Andrea},
  journal={arXiv preprint arXiv:2505.21444},
  year={2025}
}

@article{simonds2025ladder,
  title={Ladder: Self-improving llms through recursive problem decomposition},
  author={Simonds, Toby and Yoshiyama, Akira},
  journal={arXiv preprint arXiv:2503.00735},
  year={2025}
}

@article{zhao2025learning,
  title={Learning to reason without external rewards},
  author={Zhao, Xuandong and Kang, Zhewei and Feng, Aosong and Levine, Sergey and Song, Dawn},
  booktitle={The Thirteenth International Conference on Learning Representations},
  year={2026},
  url={https://openreview.net/forum?id=OU9nFEYR2M}
}

@article{shao2025deepseekmath,
  title={Deepseekmath-v2: Towards self-verifiable mathematical reasoning},
  author={Shao, Zhihong and Luo, Yuxiang and Lu, Chengda and Ren, ZZ and Hu, Jiewen and Ye, Tian and Gou, Zhibin and Ma, Shirong and Zhang, Xiaokang},
  journal={arXiv preprint arXiv:2511.22570},
  year={2025}
}

@article{shao2024deepseekmath,
  title={Deepseekmath: Pushing the limits of mathematical reasoning in open language models},
  author={Shao, Zhihong and Wang, Peiyi and Zhu, Qihao and Xu, Runxin and Song, Junxiao and Bi, Xiao and Zhang, Haowei and Zhang, Mingchuan and Li, YK and others},
  journal={arXiv preprint arXiv:2402.03300},
  year={2024}
}

@article{guo2025deepseek,
  title={DeepSeek-R1 incentivizes reasoning in LLMs through reinforcement learning},
  author={Guo, Daya and Yang, Dejian and Zhang, Haowei and Song, Junxiao and Wang, Peiyi and Zhu, Qihao and Xu, Runxin and Zhang, Ruoyu and Ma, Shirong and Bi, Xiao and others},
  journal={Nature},
  volume={645},
  number={8081},
  pages={633--638},
  year={2025},
  publisher={Nature Publishing Group UK London}
}

@article{hubert2025olympiad,
  title={Olympiad-level formal mathematical reasoning with reinforcement learning},
  author={Hubert, Thomas and Mehta, Rishi and Sartran, Laurent and Horv{\'a}th, Mikl{\'o}s Z and {\v{Z}}u{\v{z}}i{\'c}, Goran and Wieser, Eric and Huang, Aja and Schrittwieser, Julian and Schroecker, Yannick and Masoom, Hussain and others},
  journal={Nature},
  pages={1--3},
  year={2025},
  publisher={Nature Publishing Group UK London}
}

@inproceedings{akter2026nemotron,
  title={Nemotron-crossthink: Scaling self-learning beyond math reasoning},
  author={Akter, Syeda Nahida and Prabhumoye, Shrimai and Novikov, Matvei and Han, Seungju and Lin, Ying and Bakhturina, Evelina and Nyberg, Eric and Choi, Yejin and Patwary, Mostofa and Shoeybi, Mohammad and others},
  booktitle={Proceedings of the 19th Conference of the European Chapter of the Association for Computational Linguistics (Volume 1: Long Papers)},
  pages={984--1002},
  year={2026}
}

@article{dong2025reinforcement,
  title={Reinforcement pre-training},
  author={Dong, Qingxiu and Dong, Li and Tang, Yao and Ye, Tianzhu and Sun, Yutao and Sui, Zhifang and Wei, Furu},
  journal={arXiv preprint arXiv:2506.08007},
  year={2025}
}

@article{hatamizadeh2025rlp,
  title={Rlp: Reinforcement as a pretraining objective},
  author={Hatamizadeh, Ali and Akter, Syeda Nahida and Prabhumoye, Shrimai and Kautz, Jan and Patwary, Mostofa and Shoeybi, Mohammad and Catanzaro, Bryan and Choi, Yejin},
  journal={arXiv preprint arXiv:2510.01265},
  year={2025}
}

@article{she2025dupo,
  title={DuPO: Enabling Reliable LLM Self-Verification via Dual Preference Optimization},
  author={She, Shuaijie and Bao, Yu and Lu, Yu and Xu, Lu and Li, Tao and Zhu, Wenhao and Huang, Shujian and Cheng, Shanbo and Lu, Lu and Wang, Yuxuan},
  journal={arXiv preprint arXiv:2508.14460},
  year={2025}
}

@inproceedings{songmind,
  title={Mind the Gap: Examining the Self-Improvement Capabilities of Large Language Models},
  author={Song, Yuda and Zhang, Hanlin and Eisenach, Carson and Kakade, Sham M and Foster, Dean and Ghai, Udaya},
  booktitle={The Thirteenth International Conference on Learning Representations},
  year={2025}
}

@article{wu2024reft,
  title={Reft: Representation finetuning for language models},
  author={Wu, Zhengxuan and Arora, Aryaman and Wang, Zheng and Geiger, Atticus and Jurafsky, Dan and Manning, Christopher D and Potts, Christopher},
  journal={Advances in Neural Information Processing Systems},
  volume={37},
  pages={63908--63962},
  year={2024}
}

@article{prabakaran2026quantifying,
  title={Quantifying uncertainty in protein representations across models and tasks},
  author={Prabakaran, R and Bromberg, Yana},
  journal={Nature Methods},
  pages={1--9},
  year={2026},
  publisher={Nature Publishing Group US New York}
}

@inproceedings{kuhnsemantic,
  title={Semantic Uncertainty: Linguistic Invariances for Uncertainty Estimation in Natural Language Generation},
  author={Kuhn, Lorenz and Gal, Yarin and Farquhar, Sebastian},
  booktitle={The Eleventh International Conference on Learning Representations},
  year={2023}
}

@article{farquhar2024detecting,
  title={Detecting hallucinations in large language models using semantic entropy},
  author={Farquhar, Sebastian and Kossen, Jannik and Kuhn, Lorenz and Gal, Yarin},
  journal={Nature},
  volume={630},
  number={8017},
  pages={625--630},
  year={2024},
  publisher={Nature Publishing Group UK London}
}

@article{zhang2024edt,
  title={Edt: Improving large language models' generation by entropy-based dynamic temperature sampling},
  author={Zhang, Shimao and Bao, Yu and Huang, Shujian},
  journal={arXiv preprint arXiv:2403.14541},
  year={2024}
}

@inproceedings{zhu2024hot,
  title={Hot or cold? adaptive temperature sampling for code generation with large language models},
  author={Zhu, Yuqi and Li, Jia and Li, Ge and Zhao, YunFei and Jin, Zhi and Mei, Hong},
  booktitle={Proceedings of the AAAI Conference on Artificial Intelligence},
  volume={38},
  number={1},
  pages={437--445},
  year={2024}
}

@inproceedings{yang2024exploring,
  title={Exploring compositional generalization of large language models},
  author={Yang, Haoran and Lu, Hongyuan and Lam, Wai and Cai, Deng},
  booktitle={Proceedings of the 2024 Conference of the North American Chapter of the Association for Computational Linguistics: Human Language Technologies (Volume 4: Student Research Workshop)},
  pages={16--24},
  year={2024}
}

@article{kirk2023survey,
  title={A survey of zero-shot generalisation in deep reinforcement learning},
  author={Kirk, Robert and Zhang, Amy and Grefenstette, Edward and Rockt{\"a}schel, Tim},
  journal={Journal of Artificial Intelligence Research},
  volume={76},
  pages={201--264},
  year={2023}
}

@inproceedings{loshchilovdecoupled,
  title={Decoupled Weight Decay Regularization},
  author={Loshchilov, Ilya and Hutter, Frank},
  booktitle={International Conference on Learning Representations}
}

@inproceedings{ziebart2008maximum,
  title={Maximum entropy inverse reinforcement learning.},
  author={Ziebart, Brian D and Maas, Andrew L and Bagnell, J Andrew and Dey, Anind K and others},
  booktitle={Aaai},
  volume={8},
  pages={1433--1438},
  year={2008},
  organization={Chicago, IL, USA}
}

@inproceedings{haarnoja2018soft,
  title={Soft actor-critic: Off-policy maximum entropy deep reinforcement learning with a stochastic actor},
  author={Haarnoja, Tuomas and Zhou, Aurick and Abbeel, Pieter and Levine, Sergey},
  booktitle={International conference on machine learning},
  pages={1861--1870},
  year={2018},
  organization={Pmlr}
}

@inproceedings{haarnoja2017reinforcement,
  title={Reinforcement learning with deep energy-based policies},
  author={Haarnoja, Tuomas and Tang, Haoran and Abbeel, Pieter and Levine, Sergey},
  booktitle={International conference on machine learning},
  pages={1352--1361},
  year={2017},
  organization={PMLR}
}

@misc{esm2024cambrian,
  author = {{ESM Team}},
  title = {ESM Cambrian: Revealing the mysteries of proteins with unsupervised learning},
  year = {2024},
  publisher = {EvolutionaryScale Website},
  url = {https://evolutionaryscale.ai/blog/esm-cambrian},
  urldate = {2024-12-04}
}

@article{rao2026generalist,
  title={Generalist biological artificial intelligence in modeling the language of life},
  author={Rao, Vishwanatha M and Zhang, Serena and Plosky, Brian S and Hsu, Patrick D and Wang, Bo and Zou, James and Zitnik, Marinka and Topol, Eric J and Rajpurkar, Pranav},
  journal={Nature Biotechnology},
  pages={1--16},
  year={2026},
  publisher={Nature Publishing Group US New York}
}

@article{song2025evaluating,
  title={Evaluating large language models in scientific discovery},
  author={Song, Zhangde and Lu, Jieyu and Du, Yuanqi and Yu, Botao and Pruyn, Thomas M and Huang, Yue and Guo, Kehan and Luo, Xiuzhe and Qu, Yuanhao and Qu, Yi and others},
  journal={arXiv preprint arXiv:2512.15567},
  year={2025}
}

@article{ferruz2022controllable,
  title={Controllable protein design with language models},
  author={Ferruz, Noelia and H{\"o}cker, Birte},
  journal={Nature Machine Intelligence},
  volume={4},
  number={6},
  pages={521--532},
  year={2022},
  publisher={Nature Publishing Group UK London}
}

@inproceedings{welleckneural,
  title={Neural Text Generation With Unlikelihood Training},
  author={Welleck, Sean and Kulikov, Ilia and Roller, Stephen and Dinan, Emily and Cho, Kyunghyun and Weston, Jason},
  booktitle={International Conference on Learning Representations},
  year={2020}
}

@inproceedings{holtzmancurious,
  title={The Curious Case of Neural Text Degeneration},
  author={Holtzman, Ari and Buys, Jan and Du, Li and Forbes, Maxwell and Choi, Yejin},
  booktitle={International Conference on Learning Representations},
  year={2020}
}

@article{chen2021evaluating,
  title={Evaluating large language models trained on code},
  author={Chen, Mark and Tworek, Jerry and Jun, Heewoo and Yuan, Qiming and Pinto, Henrique Ponde De Oliveira and Kaplan, Jared and Edwards, Harri and Burda, Yuri and Joseph, Nicholas and Brockman, Greg and others},
  journal={arXiv preprint arXiv:2107.03374},
  year={2021}
}
\bibliographystyle{unsrtnat}

\newpage
\appendix
\onecolumn

\section{Discussion and Limitations}\label{sec:limitation}

While our unsupervised reward optimization framework demonstrates strong performance in steering protein language models, several limitations remain. 
First, our unsupervised metrics, though effective for both ProGen2 and ESM3 on Func2Seq tasks, underperform structure-based baselines like pLDDT and pTM for inverse folding tasks (despite being much more computationally efficient as discussed in~\Cref{sec:benchmark_rewards}). Second, we only conduct \textit{in silico} evaluations in this study, and wet-lab or clinical validations are required to demonstrate the real-world impact of our proposed framework in biological applications. Finally, the reliance on proxy rewards inherently introduces approximation error compared to oracle supervision, which may limit performance in domains where ground-truth labels are available (\Cref{tab:performance,fig:pass@k}). Future work may investigate more sophisticated reward design, uncertainty-aware optimization, and hybrid approaches that combine unsupervised signals with minimal human feedback.

\textbf{Broader impacts.}  The proposed framework may support beneficial protein engineering by reducing reliance on costly experimental labels, potentially accelerating applications in enzyme design, therapeutics, vaccines, and biomaterials. At the same time, improved controllability of PLMs is a dual-use capability: similar techniques could be misused to search for proteins with harmful biological activity or to optimize pathogen-, toxin-, or immune-evasion-related functions. Since our work is evaluated only \textit{in silico}, any deployment should be coupled with biosecurity screening, restrictions on sensitive targets, access control, request auditing, and institutional biosafety/ethics review.

\section{Proofs and Derivations}\label[appendix]{app:proof}

\subsection{Proof of \Cref{thm:reward_optimization}}

\begin{proof}
Let

\begin{align}
   \mathcal{V}_{\pi_{\theta}} &= \mathbb{E}_{x\sim\mathcal{D},y\sim\pi_{\theta}(y|x)}[r(x, y)] - \beta D_{\text{KL}}[\pi_{\theta}(y|x)\,\|\,\pi_{\text{ref}}(y|x)]\label{eqn:rl_obj}\\
    \mathcal{V}^* &\triangleq  \mathbb{E}_{x\sim\mathcal{D},y\sim\pi^*(y|x)}[r(x, y)] - \beta D_{\text{KL}}[\pi^*(y|x)\,\|\,\pi_{\text{ref}}(y|x)],\label{eqn:optimal_obj}
\end{align}

From the definition of $\pi^*$ in Eq.~\ref{eqn:optimal_policy}, take log

\begin{align}
    \log\pi^*(y|x) = \log\pi_{\text{ref}}(y|x) + \frac{r(x, y)}{\beta} - \log Z(x)  
\end{align}

and rearrange

\begin{align}
    r(x, y) = \beta\log\frac{\pi^*(y|x)}{\pi_{\text{ref}}(y|x)} + \beta\log Z(x).
\end{align}

Substitute into $\mathcal{V}_{\pi_{\theta}}$ in Eq.~\ref{eqn:rl_obj}:

\begin{align}
    \mathcal{V}_{\pi_{\theta}} &= \mathbb{E}_x\mathbb{E}_{y\sim\pi_{\theta}}\left[\beta\log\frac{\pi^*}{\pi_{\text{ref}}} + \beta\log Z(x) - \beta\log\frac{\pi_{\theta}}{\pi_{\text{ref}}}\right]\nonumber\\
    &= \beta\mathbb{E}_x\left[\log Z(x) - \mathbb{E}_{y\sim\pi_{\theta}(\cdot | x)}\left[\log\frac{\pi_{\theta}(y|x)}{\pi^*(y|x)}\right] \right]\nonumber\\
    &= \underbrace{\beta\mathbb{E}_x[\log Z(x)]}_{\mathcal{V}^*} - \beta\mathbb{E}_x[D_{\text{KL}}(\pi_{\theta}(\cdot|x) || \pi^*(\cdot|x))]\label{eqn:rl_obj_2}
\end{align}

\textbf{Remark:} The first term $\mathcal{V}^*$ is exactly the value of the optimal policy:

\begin{align}
    \mathcal{V}^* = \mathbb{E}_x\mathbb{E}_{y\sim\pi^*}\left[r(x, y) - \beta\log\frac{\pi^*}{\pi_{\text{ref}}}\right],
\end{align}

which can be verified by plugging $\pi^*$ into the original objective in Eq.~\ref{eqn:optimal_obj}.

\end{proof}

\subsection{Proof of Theorem~\ref{thm:SRO_optimality}}
\begin{proof}
We minimize the population objective
\[
\mathcal{L}_{\text{SRO}} = \mathbb{E}_{x\sim\mathcal{D}}\mathbb{E}_{y\sim\pi_{\text{ref}}(y|x)}\left[\frac{\pi_{\theta}(x,y)}{\pi_{\text{ref}}(x,y)}\left(\log\frac{\pi_{\theta}(x,y)}{\pi_{\text{ref}}(x,y)} - \frac{r(x, y)}{\beta} - \alpha\right)\right],
\]
where we write $\pi_\theta(x,y) = \pi_\theta(y|x)\,p_{\mathcal{D}}(x)$ and similarly for $\pi_{\text{ref}}$, so that the joint densities share the same marginal $p_{\mathcal{D}}(x)$. Since the expectation over $x$ is fixed, we can optimize the integrand pointwise for each $(x,y)$.

Define the pointwise loss for a fixed $(x,y)$:
\[
\ell(\pi_\theta) = \frac{\pi_\theta(x,y)}{\pi_{\text{ref}}(x,y)} \left( \log \frac{\pi_\theta(x,y)}{\pi_{\text{ref}}(x,y)} - \frac{r(x,y)}{\beta} - \alpha \right).
\]
Let $u = \pi_\theta(x,y) > 0$ and treat $\pi_{\text{ref}}(x,y)$ as constant. Then
\[
\ell(u) = \frac{u}{\pi_{\text{ref}}} \left( \log \frac{u}{\pi_{\text{ref}}} - \frac{r}{\beta} - \alpha \right).
\]
Differentiate w.r.t.\ $u$:
\[
\frac{d\ell}{du} = \frac{1}{\pi_{\text{ref}}} \left( \log \frac{u}{\pi_{\text{ref}}} - \frac{r}{\beta} - \alpha \right) + \frac{u}{\pi_{\text{ref}}} \cdot \frac{1}{u}
= \frac{1}{\pi_{\text{ref}}} \left( \log \frac{u}{\pi_{\text{ref}}} - \frac{r}{\beta} - \alpha + 1 \right).
\]
Set derivative to zero for optimality:
\[
\log \frac{u^*}{\pi_{\text{ref}}} - \frac{r}{\beta} - \alpha + 1 = 0
\quad \Rightarrow \quad
\log \frac{u^*}{\pi_{\text{ref}}} = \frac{r}{\beta} + \alpha - 1.
\]
Exponentiating both sides yields
\[
u^* = \pi_{\text{ref}}(x,y) \exp\!\left( \frac{r(x,y)}{\beta} + \alpha - 1 \right).
\]
Since this holds for all $(x,y)$, the optimal policy satisfies
\[
\pi_\theta^*(x,y) = \pi_{\text{ref}}(x,y) \exp\!\left( \frac{r(x,y)}{\beta} + \alpha - 1 \right),
\]
which implies the conditional form
\[
\pi_\theta^*(y|x) = \pi_{\text{ref}}(y|x) \exp\!\left( \frac{r(x,y)}{\beta} + \alpha - 1 \right).
\]
This completes the proof.
\end{proof}







\subsection{Game-Theoretic View of Binarized Reward Optimization (BRO)}\label[appendix]{app:game_BRO}

From a game-theoretic view, the reward maximization problem induced by $\mathcal{L}_{\text{BRO}}$ in~\Cref{eqn:obj_bro} can be interpreted as zero-sum game where the learned policy $\pitheta$ competes against the reference policy $\piref$ over a labeled dataset $\mathcal{D} = \mathcal{D}^+ \cup \mathcal{D}^-$. For each prompt-response pair $(x, y)$, a \emph{win} for $\pitheta$ is defined as:
\begin{itemize}
    \item assigning higher relative likelihood than $\piref$ to positive samples ($y \in \mathcal{D}^+$), i.e., $\frac{\pitheta(y|x)}{\piref(y|x)} > 1$, and
    \item assigning lower relative likelihood than $\piref$ to negative samples ($y \in \mathcal{D}^-$), i.e., $\frac{\pitheta(y|x)}{\piref(y|x)} < 1$.
\end{itemize}

Under the BT model~\cite{bradley1952rank}, the probability that $\pitheta$ wins against $\piref$ on $(x, y)$ is modeled via logistic pairwise comparison. Define the signed reward
\[
s_{\text{BRO}}(x, y) = 
\begin{cases}
+1, & y \in \mathcal{D}^+ \\
-1, & y \in \mathcal{D}^-
\end{cases}.
\]
Then the win probability is
\begin{align}
   P_{\text{win}}(x, y) &=  \sigma\left( s_{\text{BRO}}(x, y) \cdot (r_{\theta}(x, y) - r_{\text{ref}}(x, y))\right)\\
   &= \sigma\left( s_{\text{BRO}}(x, y) \cdot \beta\log \frac{\pitheta(y|x)}{\piref(y|x)} \right), 
\end{align}

where $\sigma(z) = \frac{1}{1 + e^{-z}}$ is the logistic sigmoid function.

Maximizing the expected winrate is equivalent to maximizing the log-likelihood of observed wins, which yields the BRO objective by letting $\beta=1$:
\begin{align}
    \mathcal{L}_{\text{BRO}} &= -\mathbb{E}_{(x, y) \sim \mathcal{D}} \left[ \log \sigma\left(s_{\text{BRO}}(x, y) \cdot \beta\log \frac{\pitheta(y|x)}{\piref(y|x)} \right) \right]\\
    &= -\mathbb{E}_{(x, y)\sim\mathcal{D}}\left[\mathbbm{1}\{y\in\mathcal{D}^+\}\log\frac{\pi_{\theta}}{\pi_{\text{ref}} + \pi_{\theta}} + \left(1-\mathbbm{1}\{y\in\mathcal{D}^+\}\right)\log\frac{\pi_{\text{ref}}}{\pi_{\text{ref}} + \pi_{\theta}}\right].
\end{align}

Thus, minimizing $\mathcal{L}_{\text{BRO}}$ corresponds to optimizing $\pitheta$ to dominate $\piref$ in a zero-sum competition: $\pitheta$ seeks to increase its relative likelihood on $\mathcal{D}^+$ and decrease it on $\mathcal{D}^-$, with victory governed by a Bradley--Terry (BT) ranking where the ``skill'' of a policy on $(x, y)$ is proportional to its log-probability $\log \pi(y|x)$ by~\Cref{eqn:implicit_reward}. This frames policy improvement as maximizing the empirical winrate against the reference policy under pairwise logistic comparison.

\section{Details of Prompt Sets}\label[appendix]{app:promptset}
Following~\cite{10821712}, the seven Pfam families used for our prompt sets are:

\begin{itemize}
    \item PF00002 - GPCRs
    \item PF00042 - Globins
    \item PF00125 - Core histones
    \item PF00127 - Copper binding proteins
    \item PF00257 - Dehydrins
    \item PF00262 - Calreticulins
    \item PF03668 - P-loop ATPase
\end{itemize}

To construct compositional OOD prompt sets, for each family token, we append the first 10 residues of proteins from the next family on the list above to create 7 compositional prompt templates:

\begin{itemize}
    \item \texttt{<|pf00002|>[first 10 tokens from seqs in PF00042]}
    \item \texttt{<|pf00042|>[first 10 tokens from seqs in PF00125]}
    \item \texttt{<|pf00125|>[first 10 tokens from seqs in PF00127]}
    \item \texttt{<|pf00127|>[first 10 tokens from seqs in PF00257]}
    \item \texttt{<|pf00257|>[first 10 tokens from seqs in PF00262]}
    \item \texttt{<|pf00262|>[first 10 tokens from seqs in PF03668]}
    \item \texttt{<|pf03668|>[first 10 tokens from seqs in PF00002]}
\end{itemize}

\section{Ablation Study}\label[appendix]{app:ablation}

We provide comparative analysis of various design choices of reward in~\Cref{tab:ablation_reward}. We conclude that BRO and SRO achieve consistent improvement over the base model with different reward functions across all temperatures, demonstrating its robustness to noisy supervision signals. Among which, BRO + L1-mean and SRO + $r_T$ bring the most significant gain, which are reported in~\Cref{tab:post_training_result}. 


\begin{table}[tbh]
\centering
\caption{BRO/SRO with different choices of rewards for fine-tuning Progen2-small-mix7 on Pfam700 Func2Seq task.}
\begin{tabular}{
    l
    S[table-format=1.3]
    S[table-format=1.3]
    S[table-format=1.3]
    S[table-format=1.3]
    S[table-format=1.3]
    S[table-format=1.3]
}
\toprule
Method & {$T=0.3$} & {$T=0.5$} & {$T=0.7$} & {$T=1.0$} & {$T=1.5$} & {Average} \\
\midrule
base                & 0.210 & 0.288 & 0.390 & 0.401 & 0.118 & 0.281 \\
\midrule
BRO ($r_T$)         & 0.327 & 0.406 & 0.470 & 0.453 & 0.276 & 0.386 \\
BRO (L1,~\Cref{tab:post_training_result})            & 0.407 & 0.445 & 0.494 & 0.521 & 0.298 & 0.433 \\
BRO (entropy)       & 0.249 & 0.323 & 0.414 & 0.453 & 0.266 & 0.341 \\
BRO (PCA)           & 0.228 & 0.302 & 0.359 & 0.325 & 0.129 & 0.268 \\
\midrule
SRO ($r_T$, ~\Cref{tab:post_training_result})         & 0.383 & 0.455 & 0.498 & 0.499 & 0.301 & 0.427 \\
SRO (L1)            & 0.339 & 0.384 & 0.458 & 0.504 & 0.261 & 0.389 \\
SRO (entropy)       & 0.269 & 0.350 & 0.434 & 0.480 & 0.287 & 0.364 \\
\bottomrule
\end{tabular}\label{tab:ablation_reward}
\end{table}

Additionally, we present ablation on our proposed multi-temperature sampling strategy in~\Cref{tab:ablation_sampling}. For BRO (L1, $T=0.7$) and BRO (entropy, $T=1.0$), we sample the same amount of data as the reported BRO (L1, multi-T) with 5$\times$ prompts and the best-performing metric at a single temperature according to~\Cref{fig:reward_metrics}(a). BRO trained on multi-temperature data significantly outperforms its single-temperature counterparts, highlighting the effectiveness of the proposed reward design and sampling strategy in~\Cref{sec:reward_design}.

\begin{table}[htb]
\centering
\caption{Performance comparison across temperatures (values rounded to 3 significant figures).}
\begin{tabular}{
    l
    S[table-format=1.3]
    S[table-format=1.3]
    S[table-format=1.3]
    S[table-format=1.3]
    S[table-format=1.3]
    S[table-format=1.3]
}
\toprule
Method & {$T=0.3$} & {$T=0.5$} & {$T=0.7$} & {$T=1.0$} & {$T=1.5$} & {Average} \\
\midrule
base                     & 0.210 & 0.288 & 0.390 & 0.401 & 0.118 & 0.281 \\
\midrule
BRO (L1, multi-T)           & 0.407 & 0.445 & 0.494 & 0.521 & 0.298 & 0.433 \\
BRO (L1, $T=0.7$)        & 0.279 & 0.325 & 0.393 & 0.487 & 0.348 & 0.366 \\
BRO (entropy, $T=1.0$)   & 0.225 & 0.308 & 0.406 & 0.410 & 0.206 & 0.311 \\
\bottomrule
\end{tabular}\label{tab:ablation_sampling}
\end{table}

\newpage
\section{More Experiments on Reward Design}\label[appendix]{app:reward_design}

For our reward design, we report all candidate metrics considered in the following tables, evaluated with ESMC-600M and a smaller version, ESMC-300M\footnote{Naming rules: L1-mean measures the pairwise L1 distance between the mean per-residue embeddings of generated sequences, with $d(z_1, z_2)\triangleq \|z_1-z_2\|_1$. PCA1-bos stands for the first PCA component of the BOS token embedding of the sequence, with $d(z_1, z_2) \triangleq v^\top z_1 \quad \text{where} \quad v = \arg\max_{\|v\|_2 = 1} v^\top \Sigma v$ and $\Sigma$ being the empirical covariance matrix of the centered data.}, which demonstrate consistent correlations:

\begin{table}[htbp]
\centering
\caption{Ground-truth AUROC on DRAME-Func2Seq with ESMC-300M embedding model.}
\label{tab:up000274756}
\begin{tabular}{l|cccccc}
\hline
\textbf{Metric} & \textbf{T=0} & \textbf{T=0.3} & \textbf{T=0.5} & \textbf{T=0.7} & \textbf{T=1.0} & \textbf{T=1.5} \\
\hline
cos-mean & 0.3852 & 0.4820 & 0.4917 & 0.4996 & 0.6036 & 0.5218 \\
cos-bos & 0.3451 & 0.3922 & 0.4183 & 0.4899 & 0.6461 & 0.6585 \\
cos-eos & 0.2726 & 0.3101 & 0.3456 & 0.4205 & 0.5222 & 0.5303 \\
euc-mean & 0.4488 & 0.5937 & 0.6639 & 0.6931 & 0.8029 & 0.6606 \\
euc-bos & 0.3003 & 0.3461 & 0.3831 & 0.4639 & 0.6342 & 0.6555 \\
euc-eos & 0.2499 & 0.2948 & 0.3410 & 0.4216 & 0.5302 & 0.5626 \\
L2$^2$-mean & 0.5108 & 0.6469 & 0.6914 & 0.7051 & 0.7924 & 0.6470 \\
L2$^2$-bos & 0.3420 & 0.3883 & 0.4135 & 0.4833 & 0.6331 & 0.6519 \\
L2$^2$-eos & 0.2815 & 0.3232 & 0.3611 & 0.4373 & 0.5345 & 0.5612 \\
L1-mean & 0.4682 & 0.6228 & 0.7075 & 0.7432 & 0.8583 & 0.7420 \\
L1-bos & 0.2738 & 0.3174 & 0.3661 & 0.4701 & 0.6821 & 0.6615 \\
L1-eos & 0.2435 & 0.2892 & 0.3351 & 0.4156 & 0.5194 & 0.5546 \\
Predictive Ent & 0.1319 & 0.1302 & 0.1881 & 0.3177 & 0.4836 & 0.5037 \\
Normalized Ent& 0.1385 & 0.1459 & 0.1990 & 0.3041 & 0.5323 & 0.8805 \\
Semantic Ent~\cite{kuhnsemantic} & - & - & 0.6390 & 0.6374 & 0.6763 & 0.4658 \\
PCA-1-mean & 0.6198 & 0.5294 & 0.6294 & 0.7808 & 0.7009 & 0.6463 \\
PCA-1-bos & 0.8937 & 0.8936 & 0.8728 & 0.8120 & 0.7099 & 0.5120 \\
PCA-1-eos & 0.8787 & 0.8747 & 0.8570 & 0.7956 & 0.6664 & 0.5732 \\
PCA-2-mean & 0.9114 & 0.9095 & 0.8796 & 0.7137 & 0.6202 & - \\
PCA-2-bos & 0.5678 & 0.5248 & 0.5395 & 0.5626 & 0.5333 & - \\
PCA-2-eos & 0.5263 & 0.5037 & 0.5399 & 0.5580 & 0.4298 & - \\
\hline
\end{tabular}
\end{table}

\begin{table}[htbp]
\centering
\caption{Ground-truth AUROC on Pfam700-IID-Func2Seq with ESMC-300M embedding model.}
\label{tab:pfam700_seen}
\begin{tabular}{l|cccccc}
\hline
\textbf{Metric} & \textbf{T=0} & \textbf{T=0.3} & \textbf{T=0.5} & \textbf{T=0.7} & \textbf{T=1.0} & \textbf{T=1.5} \\
\hline
cos-mean & - & 0.7627 & 0.8314 & 0.8859 & 0.6005 & 0.2071 \\
cos-bos & - & 0.7786 & 0.8313 & 0.8743 & 0.7847 & 0.4359 \\
cos-eos & - & 0.6361 & 0.6974 & 0.7726 & 0.6073 & 0.2030 \\
euc-mean & - & 0.8339 & 0.9053 & 0.9345 & 0.7827 & 0.2957 \\
euc-bos & - & 0.7780 & 0.8393 & 0.8826 & 0.8083 & 0.4737 \\
euc-eos & - & 0.6446 & 0.7185 & 0.7932 & 0.6634 & 0.2249 \\
L2$^2$-mean & - & 0.8354 & 0.8914 & 0.9224 & 0.7489 & 0.2865 \\
L2$^2$-bos & - & 0.7788 & 0.8301 & 0.8717 & 0.7890 & 0.4557 \\
L2$^2$-eos & - & 0.6550 & 0.7190 & 0.7873 & 0.6482 & 0.2254 \\
L1-mean & - & 0.8629 & 0.9259 & 0.9422 & 0.8171 & 0.5148 \\
L1-bos & - & 0.7800 & 0.8402 & 0.8803 & 0.7984 & 0.4780 \\
L1-eos & - & 0.6463 & 0.7185 & 0.7745 & 0.6357 & 0.2419 \\
Predictive Ent & - & 0.4078 & 0.5360 & 0.7291 & 0.8357 & 0.8402 \\
Normalized Ent & - & 0.3371 & 0.3851 & 0.6268 & 0.8942 & 0.9054 \\
Semantic Ent~\cite{kuhnsemantic} & - & 0.8552 & - & 0.8384 & 0.6205 & - \\
PCA-1-mean & - & 0.9734 & 0.9500 & 0.7009 & 0.8162 & 0.7532 \\
PCA-1-bos & - & 0.8853 & 0.6380 & 0.5619 & 0.5164 & 0.5820 \\
PCA-1-eos & - & 0.9375 & 0.9196 & 0.8192 & 0.5865 & 0.5841 \\
PCA-2-mean & - & 0.7146 & 0.5458 & 0.8963 & 0.5542 & 0.7971 \\
PCA-2-bos & - & 0.8851 & 0.9372 & 0.5818 & 0.6803 & 0.8801 \\
PCA-2-eos & - & 0.6223 & 0.7382 & 0.8422 & 0.7034 & 0.6819 \\
\hline
\end{tabular}
\end{table}

\begin{table}[htbp]
\centering
\caption{Ground-truth AUROC on Pfam700-OOD-Func2Seq with ESMC-300M embedding model.}
\label{tab:pfam700_unseen}
\begin{tabular}{l|cccccc}
\hline
\textbf{Metric} & \textbf{T=0} & \textbf{T=0.3} & \textbf{T=0.5} & \textbf{T=0.7} & \textbf{T=1.0} & \textbf{T=1.5} \\
\hline
cos-mean & - & 0.3313 & 0.4095 & 0.5541 & 0.5134 & 0.3053 \\
cos-bos & - & 0.3189 & 0.3907 & 0.4851 & 0.4912 & 0.3188 \\
cos-eos & - & 0.2535 & 0.3044 & 0.4000 & 0.4140 & 0.2238 \\
euc-mean & - & 0.4670 & 0.5590 & 0.6498 & 0.5992 & 0.3192 \\
euc-bos & - & 0.3153 & 0.3812 & 0.4806 & 0.4907 & 0.3140 \\
euc-eos & - & 0.2612 & 0.3151 & 0.4187 & 0.4415 & 0.2288 \\
L2$^2$-mean & - & 0.4637 & 0.5386 & 0.6360 & 0.6110 & 0.3386 \\
L2$^2$-bos & - & 0.3158 & 0.3857 & 0.4800 & 0.4926 & 0.3238 \\
L2$^2$-eos & - & 0.2735 & 0.3306 & 0.4268 & 0.4495 & 0.2343 \\
L1-mean & - & 0.5305 & 0.5979 & 0.6185 & 0.5932 & 0.4214 \\
L1-bos & - & 0.3160 & 0.3770 & 0.4660 & 0.4803 & 0.3091 \\
L1-eos & - & 0.2634 & 0.3123 & 0.3952 & 0.4130 & 0.2359 \\
Predictive Ent & - & 0.2341 & 0.3264 & 0.4731 & 0.6578 & 0.8222 \\
Normalized Ent & - & 0.2087 & 0.2209 & 0.3602 & 0.6896 & 0.8803 \\
Semantic Ent~\cite{kuhnsemantic} & - & 0.5506 & 0.5395 & 0.6066 & 0.5490 & 0.3787 \\
PCA-1-mean & - & 0.8288 & 0.8090 & 0.7536 & 0.6214 & 0.5273 \\
PCA-1-bos & - & 0.7841 & 0.7743 & 0.6462 & 0.5051 & 0.5414 \\
PCA-1-eos & - & 0.7712 & 0.7472 & 0.7095 & 0.6144 & 0.6190 \\
PCA-2-mean & - & 0.6068 & 0.6574 & 0.5620 & 0.5276 & 0.7499 \\
PCA-2-bos & - & 0.5521 & 0.5522 & 0.6378 & 0.6305 & 0.8029 \\
PCA-2-eos & - & 0.5039 & 0.5172 & 0.5604 & 0.5162 & 0.5721 \\
\hline
\end{tabular}
\end{table}

\begin{table}[htbp]
\centering
\caption{Ground-truth AUROC on DRAME-Func2Seq with ESMC-600M embedding model.}
\label{tab:up000274756_esmc600m}
\begin{tabular}{l|cccccc}
\hline
\textbf{Metric} & \textbf{T=0} & \textbf{T=0.3} & \textbf{T=0.5} & \textbf{T=0.7} & \textbf{T=1.0} & \textbf{T=1.5} \\
\hline
cos-mean & 0.4273 & 0.5565 & 0.5933 & 0.6073 & 0.7129 & 0.5986 \\
cos-bos & 0.4180 & 0.5473 & 0.6352 & 0.7005 & 0.7958 & 0.7135 \\
cos-eos & 0.2887 & 0.3764 & 0.4355 & 0.4920 & 0.6006 & 0.5240 \\
euc-mean & 0.4394 & 0.5929 & 0.6808 & 0.7211 & 0.8135 & 0.7078 \\
euc-bos & 0.3495 & 0.4752 & 0.5889 & 0.6742 & 0.7876 & 0.7156 \\
euc-eos & 0.2448 & 0.3133 & 0.3801 & 0.4506 & 0.5759 & 0.5623 \\
L2$^2$-mean & 0.5111 & 0.6581 & 0.7170 & 0.7378 & 0.8092 & 0.6957 \\
L2$^2$-bos & 0.4126 & 0.5405 & 0.6270 & 0.6912 & 0.7860 & 0.7094 \\
L2$^2$-eos & 0.2785 & 0.3559 & 0.4068 & 0.4619 & 0.5773 & 0.5625 \\
L1-mean & 0.4399 & 0.5913 & 0.6907 & 0.7432 & 0.8477 & 0.7731 \\
L1-bos & 0.3618 & 0.4975 & 0.6278 & 0.7349 & 0.8512 & 0.7113 \\
L1-eos & 0.2274 & 0.2822 & 0.3464 & 0.4300 & 0.5699 & 0.5401 \\
Predictive Ent & 0.1319 & 0.1302 & 0.1881 & 0.3177 & 0.4836 & 0.5037 \\
Normalized Ent & 0.1385 & 0.1459 & 0.1990 & 0.3041 & 0.5323 & 0.8805 \\
Semantic Ent~\cite{kuhnsemantic}  & - & 0.6085 & - & 0.6172 & 0.6309 & 0.4942 \\
PCA-1-mean & 0.6233 & 0.5311 & 0.6387 & 0.7932 & 0.6865 & 0.5581 \\
PCA-1-bos & 0.7246 & 0.5097 & 0.8958 & 0.8509 & 0.7729 & 0.6420 \\
PCA-1-eos & 0.7056 & 0.5774 & 0.7023 & 0.7958 & 0.6470 & 0.6686 \\
PCA-2-mean & 0.9118 & 0.9139 & 0.8784 & 0.6930& - & 0.7128 \\
PCA-2-bos & 0.9273 & 0.9257& 0.7568 & 0.5666 & - & 0.5788 \\
PCA-2-eos & 0.8989& 0.9089 & 0.8047 & 0.5758& - & 0.5596 \\
\hline
\end{tabular}
\end{table}

\begin{table}[htbp]
\centering
\caption{Ground-truth AUROC on Pfam700-IID-Func2Seq with ESMC-600M embedding model.}
\label{tab:pfam700_seen_esmc600m}
\begin{tabular}{l|cccccc}
\hline
\textbf{Metric} & \textbf{T=0} & \textbf{T=0.3} & \textbf{T=0.5} & \textbf{T=0.7} & \textbf{T=1.0} & \textbf{T=1.5} \\
\hline
cos-mean & - & 0.8069 & 0.8682 & 0.9079 & 0.6255 & 0.2290 \\
cos-bos & - & 0.8326 & 0.8740 & 0.9041 & 0.8048 & 0.4209 \\
cos-eos & - & 0.6925 & 0.7484 & 0.8119 & 0.5591 & 0.2252 \\
euc-mean & - & 0.8449 & 0.9121 & 0.9349 & 0.7941 & 0.3173 \\
euc-bos & - & 0.8167 & 0.8755 & 0.9099 & 0.8212 & 0.4296 \\
euc-eos & - & 0.6621 & 0.7238 & 0.8034 & 0.6390 & 0.2269 \\
L2$^2$-mean & - & 0.8536 & 0.9024 & 0.9247 & 0.7610 & 0.2998 \\
L2$^2$-bos & - & 0.8295 & 0.8708 & 0.9010 & 0.8007 & 0.4208 \\
L2$^2$-eos & - & 0.6752 & 0.7230 & 0.7952 & 0.6261 & 0.2280 \\
L1-mean & - & 0.8590 & 0.9266 & 0.9452 & 0.8358 & 0.5629 \\
L1-bos & - & 0.8212 & 0.8861 & 0.9196 & 0.8137 & 0.4367 \\
L1-eos & - & 0.6499 & 0.7158 & 0.7924 & 0.6300 & 0.2680 \\
Predictive Ent & - & 0.4078 & 0.5360 & 0.7291 & 0.8357 & 0.8402 \\
Normalized Ent & - & 0.3371 & 0.3851 & 0.6268 & 0.8942 & 0.9054 \\
Semantic Ent~\cite{kuhnsemantic}  & - & 0.8536 & 0.8602 & 0.7981 & 0.6140 & 0.3131 \\
PCA-1-mean & - & 0.9745 & 0.9513 & 0.6758 & 0.7889& 0.7080\\
PCA-1-bos & - & 0.8383 & 0.7507 & 0.5445 & 0.7162 & 0.6951 \\
PCA-1-eos & - & 0.9496 & 0.9255 & 0.6443 & 0.5278 & 0.6219 \\
PCA-2-mean & - & 0.5816& 0.5897 & 0.8973 & 0.6205 & 0.6484 \\
PCA-2-bos & - & 0.9760 & 0.9654 & 0.8626 & 0.6660 & 0.8669\\
PCA-2-eos & - & 0.5739 & 0.8063 & 0.8622 & 0.7018 & 0.7693 \\
\hline
\end{tabular}
\end{table}

\begin{table}[htbp]
\centering
\caption{Ground-truth AUROC on Pfam700-OOD-Func2Seq with ESMC-600M embedding model.}
\label{tab:pfam700_unseen_esmc600m}
\begin{tabular}{l|cccccc}
\hline
\textbf{Metric} & \textbf{T=0} & \textbf{T=0.3} & \textbf{T=0.5} & \textbf{T=0.7} & \textbf{T=1.0} & \textbf{T=1.5} \\
\hline
cos-mean & - & 0.3936 & 0.4685 & 0.5935 & 0.5352 & 0.3302 \\
cos-bos & - & 0.4224 & 0.4841 & 0.5803 & 0.6110 & 0.3445 \\
cos-eos & - & 0.2704 & 0.3297 & 0.4528 & 0.4173 & 0.2503 \\
euc-mean & - & 0.4920 & 0.5792 & 0.6529 & 0.6126 & 0.3414 \\
euc-bos & - & 0.4104 & 0.4809 & 0.5806 & 0.6057 & 0.3317 \\
euc-eos & - & 0.2539 & 0.3026 & 0.4361 & 0.4438 & 0.2455 \\
L2$^2$-mean & - & 0.4993 & 0.5643 & 0.6414 & 0.6206 & 0.3566 \\
L2$^2$-bos & - & 0.4152 & 0.4780 & 0.5763 & 0.6099 & 0.3447 \\
L2$^2$-eos & - & 0.2581 & 0.3077 & 0.4382 & 0.4535 & 0.2499 \\
L1-mean & - & 0.5224 & 0.5990 & 0.6257 & 0.6174 & 0.4602 \\
L1-bos & - & 0.4298 & 0.5035 & 0.5819 & 0.5966 & 0.3311 \\
L1-eos & - & 0.2321 & 0.2731 & 0.3936 & 0.4050 & 0.2518 \\
Predictive Ent & - & 0.2341 & 0.3264 & 0.4731 & 0.6578 & 0.8222 \\
Normalized Ent & - & 0.2087 & 0.2209 & 0.3602 & 0.6896 & 0.8803 \\
Semantic Ent~\cite{kuhnsemantic}  & - & 0.5583 & 0.5854 & 0.5866 & 0.5625 & 0.4088 \\
PCA-1-mean & - & 0.8324 & 0.8125 & 0.7470 & 0.6262 & 0.5642\\
PCA-1-bos & - & 0.8429 & 0.7992 & 0.6334 & 0.6336 & 0.6769\\
PCA-1-eos & - & 0.7634 & 0.7529 & 0.7180 & 0.5097& 0.5947\\
PCA-2-mean & - & 0.5546 & 0.6030 & 0.5601 & 0.5207 & 0.6780 \\
PCA-2-bos & - & 0.5119 & 0.7333 & 0.6406 & 0.5619 & 0.5677 \\
PCA-2-eos & - & 0.5724 & 0.5217 & 0.5634 & 0.7023 & 0.6086 \\
\hline
\end{tabular}
\end{table}

\section{More Experiments on Post-training}\label[appendix]{app:post-training}

In parallel to~\Cref{tab:post_training_result}, we also provide keyword recovery statistics for each Pfam family prompts in~\Cref{tab:family-wise-stat} where BRO/SRO consistently achieve the best/second best performance across 7 families. Additionally, we also examine the sensitivity of SRO to hyperparameter $\beta$ in~\Cref{eqn:SRO}. ~\Cref{tab:hyperparameter} demonstrates that the performance of SRO is stable across different values of $\beta$ and outperforms the batch-normalized weight DPO (wDPO) proposed in~\cite{widatalla2024aligning}.

\begin{table}[ht]
\centering
\caption{Keyword recovery for each Pfam family averaged over five sampling temperatures.}
\resizebox{\linewidth}{!}{%
\begin{tabular}{lcccccccc}
\toprule
 & PF00002 & PF00042 & PF00125 & PF00127 & PF00257 & PF00262 & PF03668 & Overall \\
\midrule
151M model & & & & & & & & \\
\midrule
Progen2-small-mix7 & 0.359 & 0.168 & 0.295 & 0.439 & 0.156 & 0.146 & 0.406 & 0.281 \\
Progen2-small-mix7 DPO & 0.502 & 0.213 & 0.328 & 0.516 & \underline{0.237} & 0.158 & 0.494 & 0.350 \\
Progen2-small-mix7 KTO & 0.494 & 0.197 & 0.335 & 0.494 & 0.165 & 0.169 & \underline{0.503} & 0.337 \\
Progen2-small-mix7 BRO & \textbf{0.630} & \underline{0.238} & \textbf{0.503} & \underline{0.672} & \textbf{0.279} & \underline{0.180} & \textbf{0.528} & \textbf{0.433} \\
Progen2-small-mix7 SRO & \underline{0.617} & \textbf{0.266} & \underline{0.468} & \textbf{0.735} & 0.197 & \textbf{0.210} & 0.495 & \underline{0.427} \\
Progen2-small-mix7 Oracle & 0.632 & 0.432 & 0.513 & 0.776 & 0.365 & 0.304 & 0.397 & 0.488 \\
\midrule
764M model & & & & & & & & \\
\midrule
Progen2-medium-mix7 & 0.270 & 0.187 & 0.497 & 0.377 & 0.223 & 0.535 & 0.604 & 0.385 \\
Progen2-medium-mix7 DPO & 0.396 & 0.275 & 0.644 & 0.496 & \textbf{0.350} & 0.523 & 0.675 & 0.480 \\
Progen2-medium-mix7 KTO & 0.385 & 0.245 & 0.661 & 0.527 & 0.282 & 0.681 & 0.751 & 0.505 \\
Progen2-medium-mix7 BRO & \textbf{0.539} & \textbf{0.327} & \textbf{0.839} & \underline{0.691} & 0.233 & \textbf{0.907} & \underline{0.754} & \textbf{0.613} \\
Progen2-medium-mix7 SRO & \underline{0.482} & \underline{0.315} & \underline{0.838} & \textbf{0.693} & \underline{0.292} & \underline{0.827} & \textbf{0.766} & \underline{0.602} \\
Progen2-medium-mix7 Oracle & 0.276 & 0.457 & 0.812 & 0.578 & 0.502 & 0.824 & 0.774 & 0.603 \\
\bottomrule
\end{tabular}\label{tab:family-wise-stat}
}
\end{table}

\begin{table}[tbh]
\centering
\caption{SRO with different $\beta$ for fine-tuning Progen2-small-mix7 on Pfam700 Func2Seq task.}
\begin{tabular}{
    l
    S[table-format=1.3]
    S[table-format=1.3]
    S[table-format=1.3]
    S[table-format=1.3]
    S[table-format=1.3]
    S[table-format=1.3]
}
\toprule
Method & {$T=0.3$} & {$T=0.5$} & {$T=0.7$} & {$T=1.0$} & {$T=1.5$} & {Average} \\
\midrule
base                & 0.210 & 0.288 & 0.390 & 0.401 & 0.118 & 0.281 \\
\midrule
SRO ($\beta=1$, ~\Cref{tab:post_training_result})         & 0.383 & 0.455 & 0.498 & 0.499 & 0.301 & 0.427 \\
SRO ($\beta=0.5$)            & 0.374 & 0.433 & 0.483 & 0.492 & 0.308 & 0.418 \\
SRO ($\beta=2.0$)            & 0.329 & 0.448 & 0.508 & 0.506 & 0.286 & 0.415 \\
Weighted DPO~\cite{widatalla2024aligning}            & 0.455 & 0.463 & 0.455 & 0.377 & 0.153 & 0.381\\
\bottomrule
\end{tabular}\label{tab:hyperparameter}
\end{table}

\section{Details of Model Training and Evaluation}\label[appendix]{app:model_training}

 To reduce noise in the unsupervised rewards, for each prompt, we keep only the top-$4$ and bottom-$4$ of the 64 generated samples in $\mathcal{D}$ for all post-training experiments (including single-temperature control groups in~\Cref{app:ablation}), resulting in $56k$ samples for training and $28k$ for testing. The reward statistics of this selected subset are consistent with those of the full dataset, yet exhibit a stronger correlation with the ground truth (see~\Cref{tab:top4_bottom4}). For BRO/DPO/KTO that operate on binarized reward, we label top-$4$ as positive samples and bottom-$4$ as negative ones. For training DPO specifically, we randomly pair the top-$4$ and bottom-$4$ samples to produce four preference pairs per prompt. We train all BRO/SRO variants as well as two offline baselines, DPO~\cite{rafailov2023direct} and KTO~\cite{ethayarajh2024model} on the same dataset for $1$ epoch with AdamW optimizer~\cite{loshchilovdecoupled} and learning rate \texttt{5e-7} without extensive hyperparameter search on a single NVIDIA A40 GPU. Due to the memory constraint, we train the 151M model with batch size 32 and the 764M model with batch size 8. All fine-tuned models employ a top-$k$ decoding with $k=15$. The DPO and KTO baselines are trained with default configurations from the \texttt{TRL} package. All SFT and post-training are conducted via an 8:2 train-validation split.

 \begin{table}[ht]
\centering
\caption{Promptwise AUROC on selected test sets with ESMC-300M embedding model.}
\begin{tabular}{lcccc}
\toprule
 & $T=0.3$ & $T=0.5$ & $T=0.7$ & $T=1.0$ \\
\midrule
Pfam700-OOD-Func2Seq & & & & \\
\midrule
L1-mean & 0.4949 & 0.6214 & 0.7411 & 0.5875 \\
Predictive Ent & 0.2332 & 0.3479 & 0.5360 & 0.7179 \\
Normalized Ent & 0.2149 & 0.2358 & 0.4267 & 0.7841 \\
PCA1-bos & 0.8839 & 0.8583 & 0.6051 & 0.5929 \\
\midrule
Pfam700-OOD-Func2Seq Top4-Bot4 & & & & \\
\midrule
L1-mean & 0.5590 & 0.6963 & 0.8006 & 0.6426 \\
Predictive Ent & 0.2262 & 0.3317 & 0.5429 & 0.7659 \\
Normalized Ent & 0.2072 & 0.2144 & 0.4295 & 0.8185 \\
PCA1-bos & 0.9086 & 0.8829 & 0.6114 & 0.6098 \\
\bottomrule
\end{tabular}\label{tab:top4_bottom4}
\end{table}

\section{Critical Temperature as the Best-Performing Temperature}\label[appendix]{app:critical_temp}

To support this claim in~\Cref{sec:benchmark_rewards}, we report the recovery rate of the ESM3 and Progen2-small-mix7 base model on DRAME-Func2Seq and Pfam700-Func2Seq tasks respectively:

\begin{table}[ht]
\centering
\caption{Keyword recovery on DRAME-Func2Seq and Pfam700-Func2Seq tasks over six sampling temperatures. Pfam700-IID-Func2Seq refers to the 7 Pfam family tokens (without residue tokens) used in the SFT stage of Progen2-small-mix7.}
\resizebox{\linewidth}{!}{%
\begin{tabular}{lcccccc}
\toprule
 & $T=0.0$ & $T=0.3$ & $T=0.5$ & $T=0.7$ & $T=1.0$ & $T=1.5$ \\
\midrule
DRAME-Func2Seq & 0.207 & 0.210 & 0.229 & 0.259 & 0.263 & 0.094 \\
Pfam700-IID-Func2Seq & - & 0.681 & 0.747 & 0.809 & 0.755 & 0.322 \\
Pfam700-OOD-Func2Seq & - & 0.210 & 0.288 & 0.390 & 0.401 & 0.118 \\
\bottomrule
\end{tabular}
}
\end{table}

The best-performing temperature for ESM3 on DRAME-Func2Seq and Progen2-small-mix7 on Pfam700-Func2Seq is $T\sim 1.0$ and $T\in [0.7, 1.0]$ respectively, coinciding with the critical temperatures observed in~\Cref{fig:reward_metrics}(a) and (b).

\end{document}